\documentclass[10pt,journal,cspaper,compsoc]{IEEEtran}

\usepackage{url, graphicx, subfigure}
\usepackage{amsmath}
\usepackage{amssymb}
\usepackage{fancyhdr}
\usepackage{mathrsfs}
\usepackage{epsfig}
\usepackage{bm}
\usepackage{graphics}
\usepackage{epsf}
\usepackage{multirow, paralist}
\usepackage{mathtools}
\usepackage{url}
\usepackage{algorithmic,algorithm} 
\usepackage{array}

\newcommand{\tabincell}[2]{\begin{tabular}{@{}#1@{}}#2\end{tabular}}
\newcommand{\PreserveBackslash}[1]{\let\temp=\\#1\let\\=\temp}
\newcolumntype{C}[1]{>{\PreserveBackslash\centering}p{#1}}
\newcolumntype{R}[1]{>{\PreserveBackslash\raggedleft}p{#1}}
\newcolumntype{L}[1]{>{\PreserveBackslash\raggedright}p{#1}}
\allowdisplaybreaks

\ifCLASSOPTIONcompsoc
\else
\fi

\ifCLASSINFOpdf
\else
\fi

\newtheorem{definition}{Definition}
\newtheorem{lemma}{Lemma}
\newtheorem{thm}{Theorem}

\newtheorem{Remark}{Remark}
\newtheorem{proposition}{Proposition}

\def \a {\mathbf{a}}

\def \u {\mathbf{u}}

\def \v {\mathbf{v}}

\def \w  {\mathbf{w}}

\def \x {\mathbf{x}}

\def \z {\mathbf{z}}

\def \sign {\mathrm{sign}}

\def \bqsa {\begin{eqnarray}}
\def \eqsa {\end{eqnarray}}
\def \bqs {\begin{equation}\begin{aligned}}
\def \eqs {\end{aligned}\end{equation}}

\begin{document}
\title{Robust Online Multi-Task Learning with Correlative and Personalized Structures}

\author{Peng Yang,~\thanks{P. Yang and X. Gao are with Computer, Electrical and Mathematical Sciences \& Engineering Division at King Abdullah University of Science and Technology, Saudi Arabia. E-mail:$\{$peng.yang.2, xin.gao$\}$@kaust.edu.sa}
        \and
        Peilin Zhao,~\thanks{P. Zhao is a senior algorithm expert in Ant Financial Service Group, China. E-mail: peilinzhao@hotmail.com}
        \and
        Xin Gao~\thanks{X. Gao and P. Zhao are corresponding authors.
        }
        }
\date{}

\IEEEcompsoctitleabstractindextext{
\begin{abstract}
Multi-Task Learning (MTL) can enhance a classifier's generalization performance by learning multiple related tasks simultaneously.
Conventional MTL works under the offline or batch setting, and suffers from expensive training cost and poor scalability.
To address such inefficiency issues, online learning techniques have been applied to solve MTL problems.
However, most existing algorithms of online MTL constrain task relatedness into a presumed structure via a single weight matrix, which is a strict restriction that does not always hold in practice.
In this paper, we propose a robust online MTL framework that overcomes this restriction by decomposing the weight matrix into two components: the first one captures the low-rank common structure among tasks via a nuclear norm and the second one identifies the personalized patterns of outlier tasks via a group lasso.
Theoretical analysis shows the proposed algorithm can achieve a sub-linear regret with respect to the best linear model in hindsight.
Even though the above framework achieves good performance, the nuclear norm that simply adds all nonzero singular values together may not be a good low-rank approximation.
To improve the results, we use a log-determinant function as a non-convex rank approximation. The gradient scheme is applied to optimize log-determinant function
and can obtain a closed-form solution for this refined problem. Experimental results on a number of real-world applications verify the efficacy of our method.
\end{abstract}

}

\maketitle

\IEEEdisplaynotcompsoctitleabstractindextext

\IEEEpeerreviewmaketitle

\section{Introduction}

Multi-Task Learning (MTL) aims to enhance the overall generalization performance by learning multiple related tasks simultaneously.
It has been extensively studied from various points of view~\cite{Caruana97,Evgeniou05,WATR2010,qi2010semi}.
As an example, the common tastes of users (i.e., tasks) with respect to movies (i.e., instances) can be harnessed into a movie recommender system using MTL~\cite{pan2010transfer}.
Most MTL methods run under the offline learning setting where the training data for each task is available beforehand.
However, offline learning methods are generally inefficient, since they suffer from a high training cost and poor scalability.
This is especially true when it comes to the large-scale streaming data.
As a remedy, MTL has been studied under the online setting, in which the model runs over a sequence of data by processing them one by one~\cite{anderson2008theory}.
After updating the model in each round, the current input will be discarded.
As a result, online learning algorithms are efficient and scalable, and have been successfully applied to a number of MTL applications~\cite{Saha,lugosi2009online,ruvolo2014online,attenberg2009collaborative,yang2016learning}.

In this paper, we investigate MTL under the online setting.
Existing online MTL methods assume that all tasks are related with each other and simply constrain their relationships via a presumed structure~\cite{Saha,CavallantiCG10}.
However, such a constraint may be too restrictive and rarely hold in the real-life applications, as the personalized tasks with individual traits often exist~\cite{gong2012robust}.
We attempt to address this drawback through a creative formulation of online MTL that consists of two components:
the first component captures a low-rank \emph{correlative} structure over the related tasks,
while the second one represents the \emph{personalized} patterns of individual tasks.

Specifically, our algorithm learns a weight matrix which is decomposed into two components as aforementioned.
A nuclear norm regularization is imposed on the first component to induce a low-rank \emph{correlative} structure of the related tasks.
A group lasso penalty is applied onto the second component of all individual tasks to identify the outliers.
Next, we apply an online projected gradient scheme to solve this non-smooth problem with a closed-form solution for the \emph{correlative} and \emph{personalized} components.
This gives our algorithm two advantages: 1) it is efficient to make predictions and update models in a real-time manner; 2) it can achieve a good trade-off between the common and personalized structures.
We provide a theoretical evaluation for our algorithm by giving a proof that our algorithm can achieve a sub-linear regret compared to the best linear model in hindsight.

Although our algorithm achieves good performance, it may not accurately approximate a low-rank matrix: the nuclear norm is essentially the $l_1$ norm of singular values, known for being biased in estimation since large singular values are detrimental to the approximation. To address this issue, we use a log-determinant function to approximate the matrix rank, that is able to reduce the contributions of large singular values while keeping those of small singular values close to zero. To solve this non-convex optimization problem, a proximal gradient algorithm is derived to adaptively learn such a low-rank structure with a closed-form solution.
In addition, we prove that there is a unique root of the refined objective under a proper parameter setting.
Finally, we conduct comparative experiments against a variety of state-of-the-art techniques on three real-world datasets.
Empirically, the refined algorithm with the log-determinant function achieves better performance than that with the nuclear norm due to a better low-rank approximation.

The rest of this paper is organized as follows.
Section 2 introduces related work.
The problem setting and the proposed algorithm with analysis are presented in Section 3 and Section 4, respectively.
Section 5 provides experimental results.
Section 6 concludes this paper.

\section{Related Work}

In this section, we briefly introduce works related to MTL in the offline and online settings, followed by the low-rank matrix approximation.

\subsection*{Multi-Task Learning}
Conventional offline or batch MTL algorithms can be broadly classified into the following two categories: \emph{explicit} parameter sharing and \emph{implicit} parameter sharing.
In the first category, all tasks can be made to share some common parameters explicitly.
Such common parameters include hidden units in neural networks~\cite{baxter2000model}, prior in hierarchical Bayesian models~\cite{bakker2003task,yu2005learning}, feature mapping matrix~\cite{AndoZ05} and classification weight~\cite{Evgeniou04}.
On the other side, the shared structure can be estimated in an implicit way by imposing a low rank subspace~\cite{pong2010trace,negahban2011estimation}, e.g. Trace-norm Regularized Multi-task Learning (TRML)~\cite{zhou2012mutal} captured a common low-dimensional subspace of task relationship with a trace-norm regularization; or a common set of features~\cite{argyriou2008convex,yang2009heterogeneous}, e.g. Multi-Task Feature Learning (MTFL)~\cite{ArgyriouEP06} learned a common feature across the tasks in an unsupervised manner. Besides, \cite{AbernethyBR07} and \cite{Agarwal} proposed a few experts to learn the task relatedness on the entire task set. These MTL techniques have been successfully used in the real-world applications, e.g. multi-view action recognition~\cite{yan2014multitask}, spam detection~\cite{haideceptive}, head pose estimation~\cite{yan2016multi}, etc.

Compared to the offline learning, online learning techniques are more efficient and suitable to handle massive and sequential data~\cite{yang2015min,conf/icml/ZhaoHJY11,yang2016efficient}. An early work ~\cite{Dekel,DekelLS06}, Online Multi-Task Learning (OMTL), studied online learning of multiple tasks in parallel. It exploited the task structure by using a global loss function. Another work \cite{LiCHLJ11} proposed a collaborative online framework, Confidence-weighted Collaborative Online Multi-task Learning (CW-COL), which learned the take relativeness via combining the individual and global variations of online Passive-Aggressive (PA) algorithms~\cite{Crammer}.
Instead of fixing the task relationship via a presumed structure~\cite{CavallantiCG10}, a recent Online Multi-Task Learning approach introduced an adaptive interaction matrix which quantified the task relevance with LogDet Divergence (OMTLLOG) and von-Neumann Divergence (OMTLVON)~\cite{Saha}, respectively. Most Recently, \cite{crammer2012learning} proposed an algorithm, Shared Hypothesis model (SHAMO), which used a K-means-like procedure to cluster different tasks in order to learn the shared hypotheses.
Similar to SHAMO, \cite{murugesan2016adaptive} proposed an Online Smoothed Multi-Task Learning with Exponential updates (OSMTL-e). It jointly learned both the per-task model parameters and the inter-task relationships in an online MTL setting.
The algorithm presented in this paper differs from existing ones in that it can learn both a common structure among the correlative tasks and the individual structure of outlier tasks.

\subsection*{Low-Rank Matrix Approximation}
In many areas (e.g. machine learning, signal and image processing), high-dimensional data are commonly used. Apart from being uniformly distributed, high-dimensional data often lie on the low-dimensional structures. Recovering the low-dimensional subspace can well preserve and reveal the latent structure of the data. For example, face images of an individual under different lighting conditions span a low-dimensional subspace from an ambient high-dimensional space \cite{basri2003lambertian}. To learn low-dimensional subspaces, recently proposed methods, such as Low-Rank Representation (LRR) \cite{liu2013robust} and Low-Rank Subspace and Clustering (LRSC) \cite{favaro2011closed}, usually depended on the nuclear norm as a convex rank approximation function to seek low-rank subspaces. Unlike the rank function that treats them equally, the nuclear norm simply adds all nonzero singular value together, where the large values may contribute exclusively to the approximation, rendering it much deviated from the true rank. To resolve this problem, we propose a log-determinant function to approximate the rank function, which is able to reduce the contributions of large singular values while keeping those of small singular values close to zero. To the best of our knowledge, this is the first work that exploits a log-determinant function to learn a low-rank structure of task relationship in the online MTL problem.

\section{Problem Setting}

In this section, we first describe our notations, followed by the problem setting of online MTL.

\subsection{Notations}

Lowercase letters are used as scalars, lowercase bold letters as vectors, uppercase letters as elements of a matrix, and bold-face uppercase letters as matrices.
$\x_i$ and $X_{ij}$ denote the $i$-th column and the $(i,j)$-th element of a matrix $\mathbf{X}$, respectively.
Euclidean and Frobenius norms are denoted by $\|\cdot\|$ and $\|\cdot\|_F$, respectively.
In particular, for every $q,p \geq 1$, we define the $(q,p)$-norm of $\mathbf{A} \in \mathbb{R}^{d\times m}$ as $\|\mathbf{A}\|_{q,p} = (\sum_{i=1}^m\|\a_i\|_q^p)^{\frac{1}{p}}$.
When the function $f(\w)$ is differentiable, we denote its gradient by $\nabla f(\w)$.

\subsection{Problem Setting}

According to the online MTL setting, we are faced with $m$ different but related classification problems, also known as tasks.
Each task has a sequential instance-label pairs, i.e., $\{(\x_t^i,y_t^i)\}_{1\leq t \leq T}^{1\leq i\leq m}$, where $\x_t^i \in \mathbb{R}^d$ is a feature vector drawn from a single feature space shared by all tasks, and $y_t^i \in \{\pm 1\}$.
The algorithm maintains $m$ separate models in parallel, one for each of the $m$ tasks.
At the round $t$, $m$ instances $\{\x_t^1,\ldots,\x_t^m\}$ are presented at one time.
Given the $i$-th task instance $\x_t^i$, the algorithm predicts its label using a linear model $\z_t^i$, i.e., $\hat{y}_t^i = \textrm{sign}(\hat{p}_t^i)$, where $\hat{p}_t^i = \z_t^{i\top}\x_t^i$ and $\z_t^i$ is the weight parameter of the round $t$.
The true label $y_t^i$ is not revealed until then.
A hinge-loss function is applied to evaluate the prediction,
\bqs
f_t^i(\z_t^i) = [1 - y^i_t \hat{p}_t^i]_+= [1 - y^i_t\z_t^i\cdot\x^i_t]_{+},
\eqs
where $[a]_{+} = \max\{0,a\}$.
The cumulative loss over all $m$ tasks at the round $t$ is defined as
\bqs
F_t(Z_t) = \sum_{i=1}^m f_t^i(\z_t^i),
\eqs
where $Z_t = [\z_t^1,\ldots,\z_t^m] \in \mathbb{R}^{d\times m}$ is the weight matrix for all tasks.
Inspired by the Regularized Loss Minimization (RLM) in which one minimizes an empirical loss plus a regularization term jointly~\cite{shalev2011stochastic}, we formulate our online MTL to minimize the regret compared to the best linear model in hindsight,
\bqs\label{RegretStatic}\notag
R_\phi \triangleq \sum_{t=1}^T [F_t(Z_t) + g(Z_t)] - \inf_{Z\in\Omega}\sum_{t=1}^T [F_t(Z) + g(Z)],
\eqs
where $\Omega\subset\mathbb{R}^d$ is a closed convex subset and the regularizer $g:\Omega\rightarrow\mathbb{R}$ is a convex regularization function that constraints $\Omega$ into simple sets, e.g. hyperplanes, balls, bound constraints, etc.
For instance, $g(Z) = \|Z\|_1$ constrains $Z$ into a sparse matrix.

\vspace{-0.01in}
\begin{figure}
\centering
\caption{Learning Personalized and Low-rank Structures from Multiple Tasks}\label{Flow-Chart}
\subfigure {\includegraphics[width=0.52\textwidth,height=5.1cm]{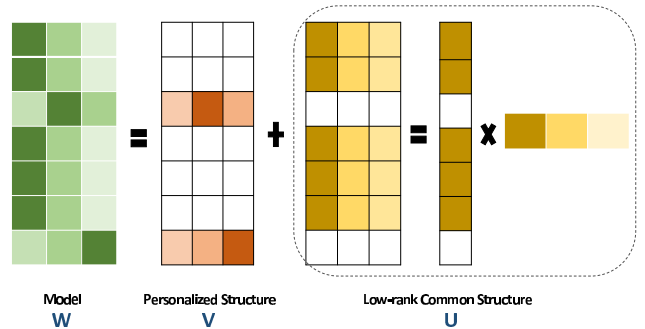}}
\end{figure}

\section{Algorithm}

We propose to solve the regret $R_\phi$ by two steps: 1) to learn the correlative and personalized patterns over multiple tasks; 2) to achieve an optimal solution for the regret $R_\phi$.

\subsection{Correlative and Personalized Structures}

We propose a novel formulation for online MTL that incorporates two components, as illustrated in Fig. \ref{Flow-Chart}. The first component captures a low-rank \emph{common} structure $U$ over the similar tasks, where one model (or pattern) can be shared cross the related tasks. As outlier tasks often exist in real-world scenarios, the second one, $V$, identifies the \emph{personalized} patterns specific to individual tasks. Thus, incorporation of two structures $U$ and $V$ could make the final model $W$ more robust and reliable.

To learn both correlative and personalized structures from multiple tasks, we decompose the weight matrix $Z$ into two components: \emph{correlative} matrix $U$ and \emph{personalized} matrix $V$, and define a new weight matrix,
\bqs
\label{decompositionW}
\Omega = \{W | W = \begin{bmatrix} U \\ V \end{bmatrix}, U\in\mathbb{R}^{d\times m}, V\in\mathbb{R}^{d\times m}\},
\eqs
where $\w^{i} = \begin{bmatrix} \u^{i} \\ \v^{i} \end{bmatrix}\in\mathbb{R}^{2d}$ is the $i$-th column of the weight matrix $W = [\w^{1},\ldots,\w^{m}]\in\mathbb{R}^{2d\times m}$.
Denoted by matrix $Z$ the summation of $U$ and $V$, we obtain
\bqs
\label{decomposeZ}
Z = U + V = \begin{bmatrix} I_d, I_d \end{bmatrix}W,
\eqs
where $I_d\in\mathbb{R}^{d\times d}$ is an identity matrix.
Given an instance $(\x^{i}_{t},y^{i}_{t})$ , the algorithm makes prediction based on both the correlative and personalized parameters,
\bqs
\hat{p}_t^i & = \z_t^{i}\cdot\x_t^i \overset{(\ref{decomposeZ})} {=} (\begin{bmatrix} I_d , I_d \end{bmatrix} \w^{i}_t) \cdot \x_t^i \\
 & = (\u_t^i+\v_t^i)^{\top}\x_t^i,
\eqs
with the corresponding loss function,
\bqs
\notag
f_t^i(\z^i_t) = f_t^i(\begin{bmatrix} I_d, I_d \end{bmatrix}\w^{i}_t) = [1 - y^{i}_t(\u^{i}_t+\v^{i}_t)^{\top}\x^{i}_t]_{+}.
\eqs
We thus can reformat the cumulative loss function with respect to $W$,
\bqs
\label{LossFunctionW}
L_t(W_t) = F_t(Z_t) \overset{(\ref{decomposeZ})} {=} F_t(\begin{bmatrix} I_d, I_d \end{bmatrix}W_t).
\eqs
We impose a regularizer on $U$ and $V$, respectively,
\bqs
\label{RegularizerW}
& r(W) =  g(\begin{bmatrix} I_d, I_d \end{bmatrix}W) \triangleq \lambda_1 r(U) + \lambda_2 r(V),
\eqs
where $\lambda_1$ and $\lambda_2$ are non-negative trade-off parameters.
Substituting Eq. (\ref{LossFunctionW}) and (\ref{RegularizerW}) into the regret $R_\phi$, it can be formatted as
\bqs
\label{RegretStaticW}
R_\phi \triangleq \sum_{t=1}^T [L_t(W_t) + r(W_t)] - \inf_{W\in\Omega}\sum_{t=1}^T [L_t(W) + r(W)],
\eqs
where $\phi_t(W) = L_t(W) + r(W)$ is a non-smooth convex function.
We next show how to achieve an optimal solution to the reformatted regret (\ref{RegretStaticW}).

\subsection{Online Task Relationship Learning}

Inspired by~\cite{bertsekas1999nonlinear}, we can solve the regret (\ref{RegretStaticW}) by a subgradient projection,
\bqs
\label{subgradient-projection}
\displaystyle{\mathop{\mathrm{argmin}}_{W\in\Omega}} \; \|W - (W_t - \eta\nabla\phi_t(W_t))\|_F,
\eqs
where $\eta > 0$ is the learning rate.
In the following lemma, we show that the problem (\ref{subgradient-projection}) can be turned into a linearized version of the proximal algorithm~\cite{rockafellar1976monotone}.
To do so, we first introduce a Bregman-like distance function~\cite{bregman1967relaxation},
\bqs
\notag
B_\psi(W,W_t) = \psi(W) - \psi(W_t) - \langle W-W_t,\nabla \psi(W_t)\rangle,
\eqs
where $\psi$ is a differentiable and convex function.
\begin{lemma}
\label{Lemma_Linearization}
Assume $\psi(\cdot)=\frac{1}{2}\|\cdot\|_F^2$, then using first-order Taylor expansion of $\phi_t(W_t)$, the algorithm (\ref{subgradient-projection}) is equivalent to a linearized form with a step-size parameter $\eta > 0$,
\bqs
\notag
\label{linearization_projection}
W_{t+1} = \displaystyle{\mathop{\mathrm{argmin}}_{W\in\Omega}} \; \langle \nabla\phi_t(W_t),W - W_t\rangle + \frac{1}{\eta}B_\psi(W,W_t).
\eqs
\end{lemma}
\noindent
Instead of balancing this trade-off individually for each of the multiple tasks, we balance it for all the tasks jointly.
However, the subgradient of a composite function, i.e. $\nabla\phi_t(W_t) = \nabla L_t(W_t) + \nabla r(W_t)$ cannot lead to a desirable effect, since we should constrain the projected gradient (i.e. $W_t - \eta\nabla\phi_t(W_t)$) into a restricted set.
To address this issue, we refine the optimization function by adding a regularizer on $W$,
\bqs
\label{optimal_solution}
W_{t+1} \triangleq \displaystyle{\mathop{\mathrm{argmin}}_{W\in\Omega}} \;
& \langle \nabla L_t(W_t),W - W_t\rangle \\
&  + \frac{1}{\eta}B_\psi(W,W_t) + r(W).
\eqs
Note that the formulation (\ref{optimal_solution}) is different from the Mirror Descent (MD) algorithm~\cite{beck2003mirror}, since we do not \emph{linearize} the regularizer $r(W)$.

Given that $W = \begin{bmatrix} U^{\top}, V^{\top} \end{bmatrix}^{\top}$, we show that the problem (\ref{optimal_solution}) can be presented with $U$ and $V$ in the lemma below.
\begin{lemma}
Assume that $\psi(W) = \frac{1}{2}\|W\|^2_F$ and $W = \begin{bmatrix} U \\ V \end{bmatrix}$, the problem (\ref{optimal_solution}) turns into an equivalent form in terms of $U$ and $V$,
\bqs
\label{OMTL-LRO_solution}
& (U_{t+1},V_{t+1}) \triangleq \displaystyle{\mathop{\mathrm{argmin}}_{U,V\in\Omega}} \; \lambda_1r(U) + \lambda_2r(V) \\
& + \frac{1}{2\eta_1}\|U - U_t\|^2_F + \frac{1}{2\eta_2}\|V - V_t\|^2_F \\
& + \langle\nabla_{U}L_t(U_t), U - U_t\rangle + \langle\nabla_{V}L_t(V_t), V - V_t\rangle,
\eqs
where the parameters $\eta_1$ and $\eta_2$ control previous learned knowledge retained by $U$ and $V$.
\end{lemma}
\begin{proof}
Assume that $\psi(W) = \frac{1}{2}\|W\|^2_F$, we obtain:
\bqs\notag
\label{B_psi_UV}
B_\psi(W,W_t) & = \frac{1}{2}\|W\|^2_F - \frac{1}{2}\|W_t\|^2_F - \langle W - W_t, W_t\rangle \\
             & = \frac{1}{2}\|W - W_t\|^2_F \\
             & \overset{(\ref{decompositionW})} {=} \frac{1}{2}\|U - U_t\|^2_F + \frac{1}{2}\|V - V_t\|^2_F.
\eqs
The linearized gradient form can be rewritten as:
\bqs\notag
\label{Linear_UV}
& \langle \nabla L_t(W_t),W - W_t \rangle
  \overset{(\ref{decompositionW})} {=}  \langle \begin{bmatrix} \nabla_{U}L_t(U_t) \\ \nabla_{V}L_t(V_t) \end{bmatrix} , \begin{bmatrix} U - U_t \\ V - V_t \end{bmatrix} \rangle \\
& =  \langle\nabla_{U}L_t(U_t), U - U_t\rangle + \langle\nabla_{V}L_t(V_t), V - V_t\rangle.
\eqs
Substituting above two inferences and (\ref{RegularizerW}) into problem (\ref{optimal_solution}), we complete this proof.
\end{proof}
We next introduce the regularization $r(U)$ and $r(V)$, and then present how to optimize this non-smooth convex problem with a closed-form solution.

\subsection{Regularization}

As mentioned above, restricting task relatedness to a presumed structure via a single weight matrix~\cite{Saha} is too strict and not always plausible in practical applications.
To overcome this problem, we thus impose a regularizer on $U$ and $V$ as follows,
\bqs\label{regularization-UV}
r(U) = \|U\|_{*} & & r(V) = \|V\|_{2,1}.
\eqs
A nuclear norm~\cite{pong2010trace} is imposed on $U$ (i.e., $\|U\|_{*}$) to represent multiple tasks ($\u_i,i\in[1,m]$) by a small number (i.e. $n$) of the basis ($n\ll m$).
Intuitively, a model performing well on one task is likely to perform well on the similar tasks.
Thus, we expect that the best model can be shared across several related tasks.
However, the assumption that all tasks are correlated may not hold in real applications.
Thus, we impose the $l_{(2,1)}$-norm~\cite{kim2010tree} on $V$ (i.e., $\|V\|_{2,1}$), which favors a few non-zero columns in the matrix $V$ to capture the personalized tasks.

Note that our algorithm with the regularization terms above is able to detect personalized patterns, unlike the algorithms~\cite{AbernethyBR07,Agarwal,chen2012learning}.
Although prior work~\cite{gong2012robust} considers detecting the personalized task, it was designed for the offline setting, which is different from our algorithm since we learn the personalized pattern adaptively with online techniques.

\subsubsection{Optimization}

Although the composite problem (\ref{OMTL-LRO_solution}) can be solved by~\cite{vandenberghe1996semidefinite}, the composite function with linear constraints has not been investigated to solve the MTL problem.
We employ a projected gradient scheme~\cite{boyd2004convex,boyd2011distributed} to optimize this problem with both smooth and non-smooth terms.
Specifically, by substituting (\ref{regularization-UV}) into (\ref{OMTL-LRO_solution}) and omitting the terms unrelated to $U$ and $V$, the problem can be rewritten as a projected gradient schema,
\bqs
\notag
(U_{t+1},V_{t+1}) = \; & \displaystyle{\mathop{\mathrm{argmin}}_{U,V\in\Omega}} \;
\frac{1}{2\eta_1}\|U - \hat{U}_t\|_F^2 + \lambda_1\|U\|_* \\
& + \frac{1}{2\eta_2}\|V - \hat{V}_t\|_F^2 + \lambda_2\|V\|_{2,1}.
\eqs
where
\bqs
\notag
\hat{U}_t = U_t - \eta_1\nabla_U L_t(U_t), & & \hat{V}_t = V_t - \eta_2\nabla_V L_t(V_t).
\eqs
Due to the decomposability of the objective function above, the solution for $U$ and $V$ can be optimized separately,
\bqs
\label{OptimalU}
U_{t+1} = \displaystyle{\mathop{\mathrm{argmin}}_{U\in\Omega}} \; \frac{1}{2\eta_1}\|U - \hat{U}_t\|_F^2 + \lambda_1\|U\|_{*}.
\eqs
\bqs
\label{OptimalV}
V_{t+1} = \displaystyle{\mathop{\mathrm{argmin}}_{V\in\Omega}} \; \frac{1}{2\eta_2}\|V - \hat{V}_t\|_F^2 + \lambda_2\|V\|_{2,1}.
\eqs
This has two advantages: 1) there is a closed-form solution for each update; 2) the update for the $U$ and $V$ can be performed in parallel.

\noindent
\textbf{Computation of U:}
\noindent
Inspired by~\cite{boyd2004convex}, we show that the optimal solution to (\ref{OptimalU}) can be obtained via solving a simple convex optimization problem in the following theorem.
\begin{thm}\label{them_U}
Denote by the eigendecomposition of $\hat{U}_t = P\hat{\Sigma} Q^{\top}\in\mathbb{R}^{d\times m}$ where $r = $\textrm{rank}$(\hat{U}_t)$, $P\in\mathbb{R}^{d\times r}$, $Q\in\mathbb{R}^{m\times r}$, and $\hat{\Sigma}=$\textrm{diag}$(\hat{\sigma}_1,\ldots,\hat{\sigma}_r)\in\mathbb{R}^{r\times r}$.
Let $\{\sigma_i\}_{i=1}^r, \sigma_i \geq 0$ be the solution of the following problem,
\bqs
\label{OptimalSigma}
\min_{\{\sigma_i\}_{i=1}^r} & \frac{1}{2\eta_1}\sum_{i=1}^r(\sigma_i - \hat{\sigma}_i)^2 +  \lambda_1 \sum_{i=1}^r \sigma_i.
\eqs
It is easy to obtain the optimal solution for (\ref{OptimalSigma}): $\sigma_i^* = [\hat{\sigma}_i - \eta_1\lambda_1]_{+}$ for $i\in[1,r]$.
Assume that $\Sigma^{*} =$diag$(\sigma_1^*,\ldots,\sigma_r^*)\in\mathbb{R}^{r\times r}$, the optimal solution to Eq. (\ref{OptimalU}) is given by,
\bqs
\label{OptimalSolutionU}
U^* = P\Sigma^{*} Q^{\top},
\eqs
\end{thm}

\noindent
\textbf{Computation of V:}
\noindent
We rewrite (\ref{OptimalV}) by solving an optimization problem for each column,
\bqs
\label{optimizationV1}
\min_{\v_i\in\mathbb{R}^d} \sum_{i=1}^m \frac{1}{2\eta_2}\|\v_i - \hat{\v}_i\|^2 + \lambda_2\sum_{i=1}^m\|\v_i\|_2.
\eqs
where $\hat{\v}_i\in\mathbb{R}^{d}$ denotes the $i$-th column of $\hat{V}_t= V_t - \eta_2\nabla_V L_t(V_t) = [\hat{\v}_1,\ldots,\hat{\v}_m]$.
The optimal operator problem (\ref{optimizationV1}) above admits a closed-form solution with time complexity of $O(dm)$~\cite{tibshirani1996regression},
\bqs
\label{OptimalSolutionV}
\v^{i}_{t+1} = \max(0,1 - \frac{\eta_2\lambda_2}{\|\hat{\v}^{i}_t\|_2})\hat{\v}^{i}_t, \quad \forall i\in[1,m].
\eqs
We observe that $\v^i$ would be retained if $\|\hat{\v}^{i}\|_2 > \eta_2\lambda_2$, otherwise, it decays to $\textbf{0}$.
Hence, we infer that only the personalized patterns among the tasks, which differs from the low-rank common structures and thus cannot be captured by $U$, would be retained in $V$.

The two quantities $V_t$ and $U_t$ can be updated according to a closed-form solution on each round $t$. A mistake-driven strategy is used to update the model.
Finally, this algorithm, which we call Robust Online Multi-tasks learning under Correlative and persOnalized structures with NuClear norm term (ROMCO-NuCl), is presented in Alg. \ref{ROMCO-NuCl}.

\begin{algorithm}[t]
\caption{ROMCO-NuCl} \label{ROMCO-NuCl}
\begin{algorithmic}[1]
\STATE {\bf Input}: a sequence of instances $(\x^{i}_t,y^{i}_t,), \forall t\in[1,T]$, and the parameter $\eta_1$, $\eta_2$, $\lambda_1$ and $\lambda_2$.
\STATE {\bf Initialize}: $\u_{0}^i = \mathbf{0}, \v_{0}^i = \mathbf{0}$ for $\forall i\in[1,m]$;
\FOR{$t=1,\ldots, T$}
    \FOR{$i=1,\ldots, m$}
        \STATE Receive instance pair ($\x^{i}_t$ $y^{i}_t$);
        \STATE Predict $\hat{y}_t^{i} = \sign[(\u^{i}_t + \v^{i}_t)\cdot\x_t^{i}]$;
        \STATE Compute the loss function $f_t^i(\w_t^{i})$;
    \ENDFOR
    \IF {$\exists i\in[1,m], f_t^i(\w_t^i) > 0$}
        \STATE Update $U_{t+1}$ with Eq.~(\ref{OptimalSolutionU});
        \STATE Update $V_{t+1}$ with Eq.~(\ref{OptimalSolutionV});
    \ELSE
        \STATE $U_{t+1} = U_{t}$ and $V_{t+1} = V_{t}$;
    \ENDIF
\ENDFOR
\STATE {\bf Output}: $\w^{i}_T$ for $i\in[1,m]$
\end{algorithmic}
\end{algorithm}

\subsection{Log-Determinant Function}

\vspace{-0.01in}
\begin{figure}
\centering
\caption{The rank, nuclear norm and log-determinant objectives in the scalar case}\label{Rank-Value-Comparison}
\subfigure {\includegraphics[width=0.52\textwidth,height=5.3cm]{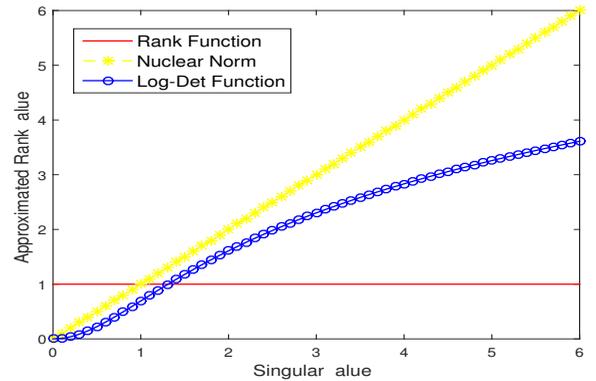}}
\end{figure}
\vspace{-0.0in}

While the nuclear norm has been theoretically proven to be the tightest convex approximation to the rank function, it is usually difficult to theoretically prove whether the nuclear norm near-optimally approximates the rank function, e.g., the incoherence property \cite{candes2009exact}\cite{candes2010power}.
In addition, the nuclear norm may not accurately approximate the rank function in practice, since the matrix "rank" regards all nonzero singular values to have equal contributions, i.e., regarding all positive singular values as ''1", as shown by the red line in Fig. \ref{Rank-Value-Comparison}; while the nuclear norm, as shown by the yellow star line in Fig. \ref{Rank-Value-Comparison}, treats the nonzero singular values differently, i.e., it simply adds all nonzero values together, thus the larger singular values make more contribution to the approximation.

To solve this issue, we introduce a log-determinant function as follows,
\begin{definition}
Let $\sigma_i$ ($i=1,2,\ldots,n$) be the singular values of the matrix $U$,
\bqs\label{log_determinant}\notag
r^{ld}(U) = \log\det\left(I + U^{\top}U\right) = \sum_{i=1}^n \log\left(1 + \sigma_i^2\right).
\eqs
\begin{itemize}
  \item When $\sigma_i = 0$, the term $\log(1+\sigma_i^2) = 0$, which is the same as the true rank function;
  \item When $0 < \sigma_i < 1$, $\log(1 + \sigma_i^2) < \sigma_i$, implying that those small singular values can be reduced further;
  \item For those large singular values $\sigma_i > 1$, $\log(1 + \sigma_i^2) \ll \sigma_i$, which is a significant reduce over large singular values.
\end{itemize}
\end{definition}
In this case, $r^{ld}(U)$ approximates the rank function better than the nuclear norm by significantly reducing the weights of large singular values, meanwhile regarding those very small singular values as noise, as presented with the blue circle line in Fig. \ref{Rank-Value-Comparison}.

\subsubsection{Optimal Solution}

Replacing the nuclear norm $\|\cdot\|_*$ with the log-determinant function $r^{ld}(U)$, the minimization of $U$ is reduced to the following problem:
\bqs\label{log_determinant_Objective}
U_{t+1} = \arg\min_{U} \frac{1}{2\eta_1}\|U - \hat{U}_t\|_F^2 + \lambda_1r^{ld}(U).
\eqs
To solve the objective function above, 
we show that the optimal solution could be obtained by solving the roots of a cubic equation in the following theorem.

\begin{thm}
Let $\hat{U}_t = P\hat{\Sigma} Q^{\top}\in\mathbb{R}^{d\times m}$ where $\hat{\Sigma}=$\textrm{diag}$(\hat{\sigma}_1,\ldots,\hat{\sigma}_r)\in\mathbb{R}^{r\times r}$ and $r = $\textrm{rank}$(\hat{U}_t)$. Let $\{\sigma_i\}_{i=1}^r, \sigma_i \geq 0$ be the solution of the following problem,
\bqs
\label{OptimalSigma_log}
\min_{\{\sigma_i\}_{i=1}^r} \sum_{i=1}^r\left[\frac{1}{2\eta_1}(\sigma_i - \hat{\sigma}_i)^2 + \lambda_1\log(1+\sigma_i^2)\right].
\eqs
Then the optimal solution to Eq. (\ref{log_determinant_Objective}), similar to Thm. \ref{them_U}, is given by
$U^* = P\Sigma^{*} Q^{\top}$, where $\Sigma^{*} =$diag$(\sigma_1^*,\ldots,\sigma_r^*)\in\mathbb{R}^{r\times r}$ and $\sigma_i^*$ is the optimal solution of (\ref{OptimalSigma_log}) .
To obtain the solution, the problem is reduced to solving the derivative of Eq. (\ref{OptimalSigma_log}) for each $\{\sigma_i\}_{i=1}^r \geq 0$ with $\rho = \eta_1\lambda_1$,
\bqs\label{derivate_objective}
\frac{1}{\rho}\sigma_i^3 - \frac{1}{\rho}\hat{\sigma}_i\sigma_i^2 + (\frac{1}{\rho} + 2)\sigma_i - \frac{1}{\rho}\hat{\sigma}_i = 0.
\eqs
\end{thm}
\noindent

\begin{algorithm}[t]
\caption{ROMCO-LogD} \label{LogD_ROMCO}
\begin{algorithmic}[1]
\STATE {\bf Input}: a sequence of instances $(\x^{i}_t,y^{i}_t,), \forall t\in[1,T]$, and the parameter $\eta_1$, $\eta_2$, $\lambda_1$ and $\lambda_2$.
\STATE {\bf Initialize}: $\u_{0}^i = \mathbf{0}, \v_{0}^i = \mathbf{0}$ for $\forall i\in[1,m]$;
\FOR{$t=1,\ldots, T$}
    \FOR{$i=1,\ldots, m$}
        \STATE Receive instance pair ($\x^{i}_t$ $y^{i}_t$);
        \STATE Predict $\hat{y}_t^{i} = \sign[(\u^{i}_t + \v^{i}_t)\cdot\x_t^{i}]$;
        \STATE Compute the loss function $f_t^i(\w_t^{i})$;
    \ENDFOR
    \IF {$\exists i\in[1,m], f_t^i(\w_t^i) > 0$}
        \STATE Update $U_{t+1}$ by solving Eq. (\ref{derivate_objective});
        \STATE Update $V_{t+1}$ with Eq. (\ref{OptimalSolutionV});
    \ELSE
        \STATE $U_{t+1} = U_{t}$ and $V_{t+1} = V_{t}$;
    \ENDIF
\ENDFOR
\STATE {\bf Output}: $\w^{i}_T$ for $i\in[1,m]$
\end{algorithmic}
\end{algorithm}

In general, the equation (\ref{derivate_objective}) has three roots. The details of root computation is given in Appendix.
In addition, in the following proposition, we prove that a unique positive root for (\ref{derivate_objective}) can be obtained in a certain parameter setting.
\begin{proposition}\label{Proposition1}
Assume that $\Theta(\sigma_i) = \frac{1}{2\rho}(\sigma_i - \hat{\sigma}_i)^2 + \log(1+\sigma_i^2)$ with $\rho = \eta_1\lambda_1$, when $\hat{\sigma}_i = 0$, $ \sigma_i = 0$, under the condition that $\hat{\sigma}_i > 0$ and $\frac{1}{\rho} > 1/4$, $\sigma_i$ located in $(0, \hat{\sigma}_i)$ is the unique positive root in cubic Eq. (\ref{derivate_objective}).
\end{proposition}
\begin{proof}
We need to minimize Eq. (\ref{OptimalSigma_log}) under the constraint of $\sigma_i \geq 0$. The derivative of $\Theta(\sigma_i)$ is
\bqs\notag
\Theta^{\prime}(\sigma_i) = \frac{2\sigma_i}{1 + \sigma_i^2} + \frac{1}{\rho}(\sigma_i - \hat{\sigma}_i),
\eqs
and the second derivative is
\bqs\notag
\Theta^{\prime\prime}(\sigma_i) = \frac{\frac{1}{\rho}\sigma_i^4 + (\frac{2}{\rho} - 2)\sigma_i^2 + (2 + \frac{1}{\rho})}{(1+\sigma_i^2)^2}.
\eqs
\emph{Case 1}: If $\hat{\sigma}_i = 0$, because $\sigma_i \geq 0$, we have $\Theta^{\prime}(\sigma_i) \geq 0$. That is, $\Theta(\sigma_i)$ is nondecreasing for any $\sigma_i \geq 0$ and strictly increasing with $\sigma_i > 0$. Thus, the minimizer of $\Theta(\sigma_i)$ is $\sigma_i^* = 0$.\\
\emph{Case 2}: If $\hat{\sigma}_i > 0$, then the roots exist only in a region of $(0,\hat{\sigma}_i)$ to let $\Theta^{\prime}(\sigma_i) = 0$, since $\Theta^{\prime}(\sigma_i)$ is monotonic with $\sigma_i$, and $\Theta^{\prime}(0) = -\frac{1}{\rho}\hat{\sigma}_i < 0$ while $\Theta^{\prime}(\hat{\sigma}_i) > 0$.
\begin{itemize}
  \item If $\frac{1}{\rho} > 1/4$, then $\Theta^{\prime\prime}(\sigma_i) > 0$ since $\frac{1}{\rho}\sigma_i^4 + (\frac{2}{\rho} - 2)\sigma_i^2 + (2 + \frac{1}{\rho}) > 0$. In this case, $\Theta(\sigma_i)$ is a strictly convex function with a unique root in $(0,\hat{\sigma}_i)$. Thus, the proposition is proven.
  \item If $0 < \frac{1}{\rho} \leq 1/4$, we determine the minimizer in the following way: Denote the set of positive roots of Eq. (\ref{derivate_objective}) by $\Omega_+$. By the first-order necessary optimality condition, the minimizer needs to be chosen from $\{0\}\cup\Omega_+$, that is, $\sigma^*_i = \arg\min_{\sigma_i\in\{0\}\cup\Omega_+}\Theta(\sigma_i)$.
\end{itemize}
\noindent
In our experiments, we initialize $\frac{1}{\rho} = 1$ and increase its value in each iteration. Therefore, when $\hat{\sigma}_i > 0$, the minimizer $\sigma^*_i\in(0, \hat{\sigma}_i)$ is the unique positive root of (\ref{derivate_objective}); when $\hat{\sigma}_i = 0$, $\sigma^*_i = 0$.
\end{proof}

We are ready to present the algorithm: RMOCO with Log-Determinant function for rank approximation, namely ROMCO-LogD, which also exploits a mistake-driven update rule. We summarize ROMCO-LogD in Alg. \ref{LogD_ROMCO}. To the best of our knowledge, this is the first work that proposes a log-determinant function to learn a low-rank structure of task relationship in the online MTL problem. In the next section, we will theoretically analyze the performance of the proposed online MTL algorithms ROMCO-NuCl/LogD.

\section{Theoretical Analysis}

We next evaluate the performance of our online algorithm ROMCO-NuCl/LogD in terms of the regret bound.
We first show the regret bound of the algorithm (\ref{optimal_solution}) and its equivalent form in the following lemma, which is essentially the same as Theorem 2 in the paper~\cite{duchi2010composite}:
\begin{lemma}
\label{General-Regret}
Let $\{W_t\}$ be updated according to (\ref{optimal_solution}).
Assume that $B_\psi(\cdot,\cdot)$ is $\alpha$-strongly convex w.r.t. a norm $\|\cdot\|_p$ and its convex conjugate $\|\cdot\|_q$ with $\frac{1}{p} + \frac{1}{q} = 1$, then for any $W^{*}\in\Omega$,
\bqs
\label{general_regret}\notag
R_\phi \leq \frac{1}{\eta}B_\psi(W^{*},W_1) + r(W_1) + \frac{\eta}{2\alpha}\sum_{t=1}^{T}\|\nabla L_t(W_t)\|_{q}^2.
\eqs
\end{lemma}
\begin{Remark}
We show that the above regret is $O(\sqrt{T})$ with respect to the best linear model in hindsight.
Suppose that the functions $F_t$ are Lipschitz continuous, then $\exists G_{q}$ such that $\max_t\|\nabla L_t(W_t)\|_{q} \leq G_{q}$.
Then we obtain:
\bqs
\notag
R_\phi \leq \frac{1}{\eta}B_\psi(W^{*}, W_1) + r(W_1) + \frac{T\eta}{2\alpha}G_{q}^2.
\eqs
We also assume that $r(W_1) = 0$.
Then by setting $\eta = \sqrt{2\alpha B_\psi(W^{*},W_1)}/(\sqrt{T}G_{q})$, we have $R_\phi \leq \sqrt{2TB_\psi(W^{*},W_1)}G_{q}/\sqrt{\alpha}$.
Given that $G_{q}$ is constant, setting $\eta \propto 1/\sqrt{T}$, we have $R_\phi = O(\sqrt{T})$.
\end{Remark}

\begin{lemma}
\label{ProjectionEqua}
The general optimization problem (\ref{optimal_solution}) is equivalent to the two step process of setting:
\bqs
\notag
& \tilde{W}_t =  \displaystyle{\mathop{\mathrm{argmin}}_{W\in\Omega}} \; \frac{1}{\eta}B_\psi(W,W_t) + \langle\nabla L_t(W_t), W \rangle, \\
& W_{t+1} =  \displaystyle{\mathop{\mathrm{argmin}}_{W\in\Omega}} \; \frac{1}{\eta}\{B_\psi(W,\tilde{W}_t) + r(W) \}.
\eqs
\end{lemma}
\begin{proof}
The optimal solution to the first step satisfies
$\nabla\psi(\tilde{W}_t) - \nabla\psi(W_t) + \eta\nabla L_t(W_t) = 0$, so that
\bqs
\label{FirstStep}
\nabla\psi(\tilde{W}_t) = \nabla\psi(W_t) - \eta\nabla L_t(W_t).
\eqs
Then look at the optimal solution for the second step.
For some $r'(W_{t+1})\in \partial r(W_{t+1})$, we have
\bqs
\label{SecondStep}
\nabla\psi(W_{t+1}) - \nabla\psi(\tilde{W}_t) + \eta r'(W_{t+1}) = 0.
\eqs
Substituting Eq.~(\ref{FirstStep}) into Eq.~(\ref{SecondStep}), we obtain
\bqs
\notag
\frac{1}{\eta}(\nabla\psi(W_{t+1}) - \nabla\psi(W_t)) + \nabla L_t(W_t) + r'(W_{t+1}) = 0,
\eqs
which satisfies the optimal solution to the one-step update of (\ref{optimal_solution}).
\end{proof}

We next show that ROMCO can achieve a sub-linear regret in the following theory.
\begin{thm}
\label{Strong-Regret}
The algorithm ROMCO (Alg.~\ref{ROMCO-NuCl} and Alg. \ref{LogD_ROMCO}) runs over a sequence of instances for each of the $m$ tasks.
Assume that $r(0) = 0$, i.e., $W_1  = 0$ and $\max_t \|\nabla L_t(W_t)\| \leq G_2$, $U, V \in \mathbb{R}^{d\times m}$, then the following inequality holds for all $W^{*}\in\Omega$,
\bqs
\label{boundAlgorithm}\notag
R_\phi  \leq \frac{1}{2\eta}\|W^{*}\|^2_F + T\eta G_{2}^2  = O(G_2\|W^{*}\|\sqrt{T}).
\eqs
\end{thm}
\begin{proof}
Let $\psi(\cdot)=\frac{1}{2}\|\cdot\|^2_F$, according to Lemma \ref{ProjectionEqua}, the solutions in subgradient projection (\ref{OptimalU}) and~\eqref{OptimalV} are equivalent to the one in form of the general optimization (\ref{optimal_solution}).
Based on Lemma \ref{General-Regret}, for any $U^{*}\in\Omega$,
\bqs
\notag
R_\phi \leq \frac{1}{\eta}B_\psi(W^{*}, W_1) + r(W_1) + \frac{\eta}{2\alpha}\sum_{t=1}^{T}\|\nabla L_t(W_t)\|^2_q.
\eqs
Because that $\psi(\cdot)=\frac{1}{2}\|\cdot\|^2_F$ (i.e., $p=q=2$), $\|\nabla L_t(W_t)\| \leq G_2$.
Assuming $W_1 = 0$, we obtain $r(W_1) = 0$ and $B_\psi(W^{*}, W_1)=\frac{1}{2}\|W^{*}\|_F^2$.
Thus,
\bqs
\notag
R_\phi \leq \frac{1}{2\eta}\|W^{*}\|^2_F + \frac{T\eta}{2}G_{2}^2.
\eqs
By setting $\eta = \frac{\|W^*\|}{\sqrt{T}G_2}$, we have $R_\phi = O(G_2\|W^{*}\|\sqrt{T})$.
\end{proof}

\section{Experimental Results}

We evaluate the performance of our algorithm on three real-world datasets.
We start by introducing the experimental data and benchmark setup, followed by discussions on the results of three practical applications.

\subsection{Data and Benchmark Setup}

\begin{table}[t]
\centering
\caption{Statistics of three datasets}
\label{statistic_data}
\scriptsize
\begin{tabular}[2.1\textwidth]{|c|c|c|c|} \hline
             & Spam Email & MHC-I  & EachMovie \\ \hline
\#Tasks      & 4          & 12     & 30        \\ \hline
\#Sample     & 7068       & 18664  & 6000      \\ \hline
\#Dimesion   & 1458       & 400    & 1783      \\ \hline
\#MaxSample  & 4129       & 3793   & 200       \\ \hline
\#MinSample  & 710        & 415    & 200       \\ \hline
\end{tabular}
\end{table}

\subsubsection{Experimental Datasets and Baseline}

We used three real-world datasets to evaluate our algorithm: \emph{Spam Email}\footnote{http://labs-repos.iit.demokritos.gr/skel/i-config/},
\emph{Human MHC-I}\footnote{http://web.cs.iastate.edu/~honavar/ailab/} and \emph{EachMovie} \footnote{http://goldberg.berkeley.edu/jester-data/}.
Table~\ref{statistic_data} summarizes the statistics of three datasets.
Each of the datasets can be converted to a list of binary-labeled instances, on which binary classifiers could be built for the applications of the three real-world scenarios: Personalized Spam Email Filtering, MHC-I Binding Prediction, and Movie Recommender System.

We compared two versions of the ROMCO algorithms with two batch learning methods: multi-task feature learning  (\emph{MTFL})~\cite{ArgyriouEP06} and trace-norm regularized multi-task learning (\emph{TRML}) \cite{zhou2012mutal}, as well as six online learning algorithms: online multi-task learning (\emph{OMTL}) \cite{DekelLS06}, online passive-aggressive algorithm (\emph{PA}) \cite{Crammer}, confidence-weighted online collaborative multi-task learning (\emph{CW-COL}) \cite{LiCHLJ11} and three recently proposed online multi-task learning algorithms: \emph{OMTLVON}, \emph{OMTLLOG} \cite{Saha} and \emph{OSMTL-e} \cite{murugesan2016adaptive}. Due to the expensive computational cost in the batch models, we modified the setting of MTFL and TRML to handle online data by periodically retraining them after observing 100 samples. All parameters for MTFL and TRML were set by default values. To further examine the effectiveness of the PA algorithm, we deployed two variations of this algorithm as described below:
\emph{PA-Global} learns a single classification model from data of all tasks;
\emph{PA-Unique} trains a personalized classifier for each task using its own data.
The parameter C was set to 1 for all related baselines and the ROMCO algorithms. Other parameters for CW-COL, OMTLVON(OMTLLOG) and OSMTL-e were tuned with a grid search on a held-out random shuffle.
The four parameters, $\eta_1$, $\eta_2$, $\lambda_1$ and $\lambda_2$ for ROMCO-NuCl/LogD, were tuned by the grid search $\{10^{-6}$,\ldots,$10^{0}\}$ on a held-out random shuffle.

\begin{table*}[t]
\centering
\caption{Cumulative error rate (\%) and F1-measure (\%) with their standard deviation in the parenthesis on the Spam Email Dataset}
\label{email_comparison_result}
\scriptsize
\begin{tabular}[2.1\textwidth] {|c|c|c|c|c|c|c|c|c|}
\hline
\multirow{2}{*}{Algorithm} & \multicolumn{2}{|c|}{\emph{User1}} & \multicolumn{2}{|c|}{\emph{User2}}  &  \multicolumn{2}{|c|}{\emph{User3}} &  \multicolumn{2}{|c|}{\emph{User4}} \\
\cline{2-9}
& Error Rate	& \tabincell{c}{Legit F1}  & Error Rate & \tabincell{c}{Legit F1} &  Error Rate & \tabincell{c}{Legit F1} & Error Rate & \tabincell{c}{Legit F1} \\
\hline\hline
MTFL				& 13.16(1.21)	& 88.61(1.06)	& 8.72(0.58)	& 94.67(0.38)  & 14.84(0.91)& 86.70(0.88) & 16.87(0.78) & 83.59(0.77)   \\ \hline
TRML				& 17.71(0.99)	& 84.61(0.88)	& 12.45(0.94)	& 92.24(0.65)  & 13.89(0.73)& 87.57(0.66) & 20.78(1.08) & 79.67(1.09)   \\ \hline
PA-Global           & 6.15(0.45)    & 94.54(0.41)   & 8.59(0.82)	& 94.72(0.51)  & 4.12(0.30)	& 96.33(0.27) & 9.75(0.61)  & 90.20(0.64)  \\ \hline
PA-Unqiue           & 5.05(0.49)    & 95.51(0.44)   & 8.28(0.85)	& 94.91(0.52)  & 3.67(0.36)	& 96.73(0.32) & 8.43(0.86)  & 91.52(0.89)  \\ \hline
CW-COL			    & 5.10(0.61)	& 95.45(0.55)	&  \textbf{6.58(0.60)}	&  \textbf{95.91(0.38)}  & 4.14(0.12)	& 96.29(0.10) & 7.95(0.73)  & 92.08(0.73)  \\ \hline 
OMTL				& 5.00(0.48)	& 95.55(0.43)	& 8.01(0.77)	& 95.07(0.48)  & 3.55(0.29)	& 96.84(0.26) & 8.24(0.73)  & 91.71(0.75)  \\ \hline
OSMTL-e             & 7.34(0.69)    & 93.46(0.59)   & 10.25(0.89)   & 93.61(0.57)  & 5.86(0.69) & 94.74(0.61) & 10.42(1.18) & 89.68(1.07)   \\ \hline
OMTLVON             & 18.88(5.70)  & 85.55(3.70)  & 19.76(0.10)   &  89.04(0.06)  & 3.54(0.30)  & 96.85(0.29) & 10.54(2.47) & 89.52(2.97)   \\ \hline
OMTLLOG             & 4.81(0.36)   & 95.73(0.32)  & 7.58(0.65)   &  95.36(0.39)  & 2.91(0.18)   & 97.41(0.16) & 7.16(0.53)  & 92.87(0.51)   \\ \hline \hline
\textbf{ROMCO-NuCl} & 4.12(0.50)    & 96.34(0.44)   & 7.06(0.49)  & 95.68(0.30)	& 2.87(0.42)    & 97.43(0.38) & 6.85(0.68)  & 93.23(0.66)   \\ \hline
\textbf{ROMCO-LogD} & \textbf{4.00(0.42)}    & \textbf{96.45(0.37)}   & 7.31(0.45)  & 95.55(0.27)	& \textbf{2.74(0.16)}  & \textbf{97.56(0.14)} &  \textbf{6.68(0.44)} &  \textbf{93.40(0.43)} \\ \hline
\end{tabular}
\end{table*}

\vspace{-0.05in}
\subsubsection{Evaluation Metric}
We evaluated the performance of the aforementioned algorithms by two metrics:
1) Cumulative error rate: the ratio of predicted errors over a sequence of instances.
It reflects the prediction accuracy of online learners.
2) F1-measure: the harmonic mean of precision and recall.
It is suitable for evaluating the learner's performance on class-imbalanced datasets.
We followed the method of \cite{yang2015aggressive} by randomly shuffling the ordering of samples for each dataset and repeating the experiment 10 times with new shuffles. The average results and its standard deviation are reported below.

\begin{table}[t]
\centering
\caption{Run-time (in seconds) for each algorithm}
\label{running_time}
\scriptsize
\begin{tabular}[2.1\textwidth]{|c|c|c|c|c|} \hline
Algorithm & Spam Email & MHC-I & EachMovie \\ \hline
TRML        & 73.553  & 361.42  & 391.218 \\ \hline
MTFL        & 78.012  & 198.90  & 302.170 \\ \hline
PA-Global   & 0.423    & 1.79    & 23.337   \\ \hline
PA-Unique   & 0.340    & 1.53    & 26.681   \\ \hline
CW-COL     & 0.86    & 4.35    & 31.002   \\ \hline
OMTL        & 26.428   & 40.314   & 85.105   \\ \hline
OSMTL-e     & 0.360    & 2.327   & 11.43   \\ \hline
OMTLVON        & 1.230  & 1.785   & 21.586   \\ \hline
OMTLLOG        & 1.145   & 1.371  & 20.232  \\ \hline \hline
\textbf{ROMCO-NuCl}      & 11.49   & 4.88   & 33.235   \\ \hline
\textbf{ROMCO-LogD}      & 10.59   & 5.352   & 28.716   \\ \hline
\end{tabular}
\label{tab3}
\end{table}
\vspace{-0.01in}

\subsection{Spam Email Filtering}
\vspace{-0.05in}

\begin{figure*}[t]
\centering
\caption{Cumulative error rate on the Email Spam dataset along the entire online learning process}
\label{EmailSpam-figure}
\subfigure {\includegraphics[width=0.2435\textwidth,height=4.4cm]{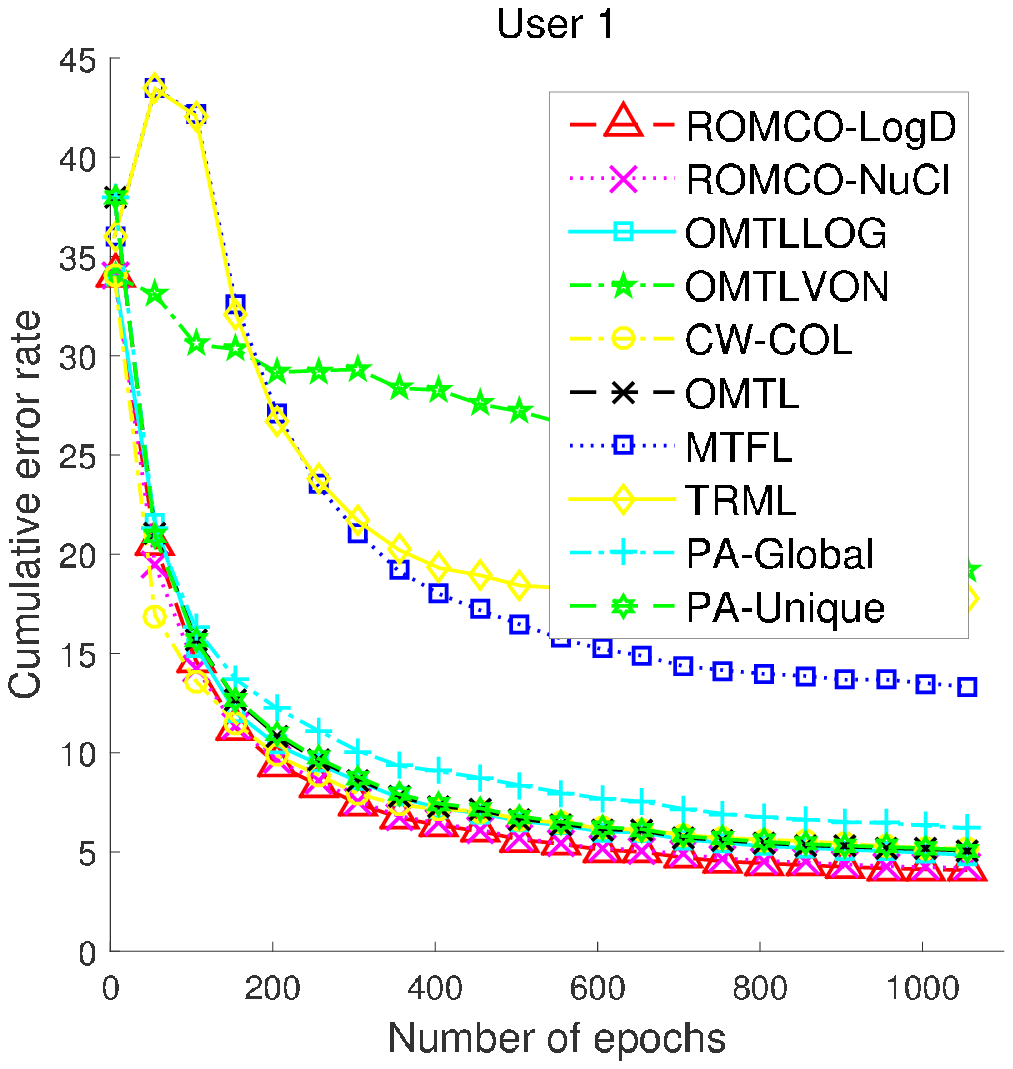}}
\subfigure {\includegraphics[width=0.2435\textwidth,height=4.4cm]{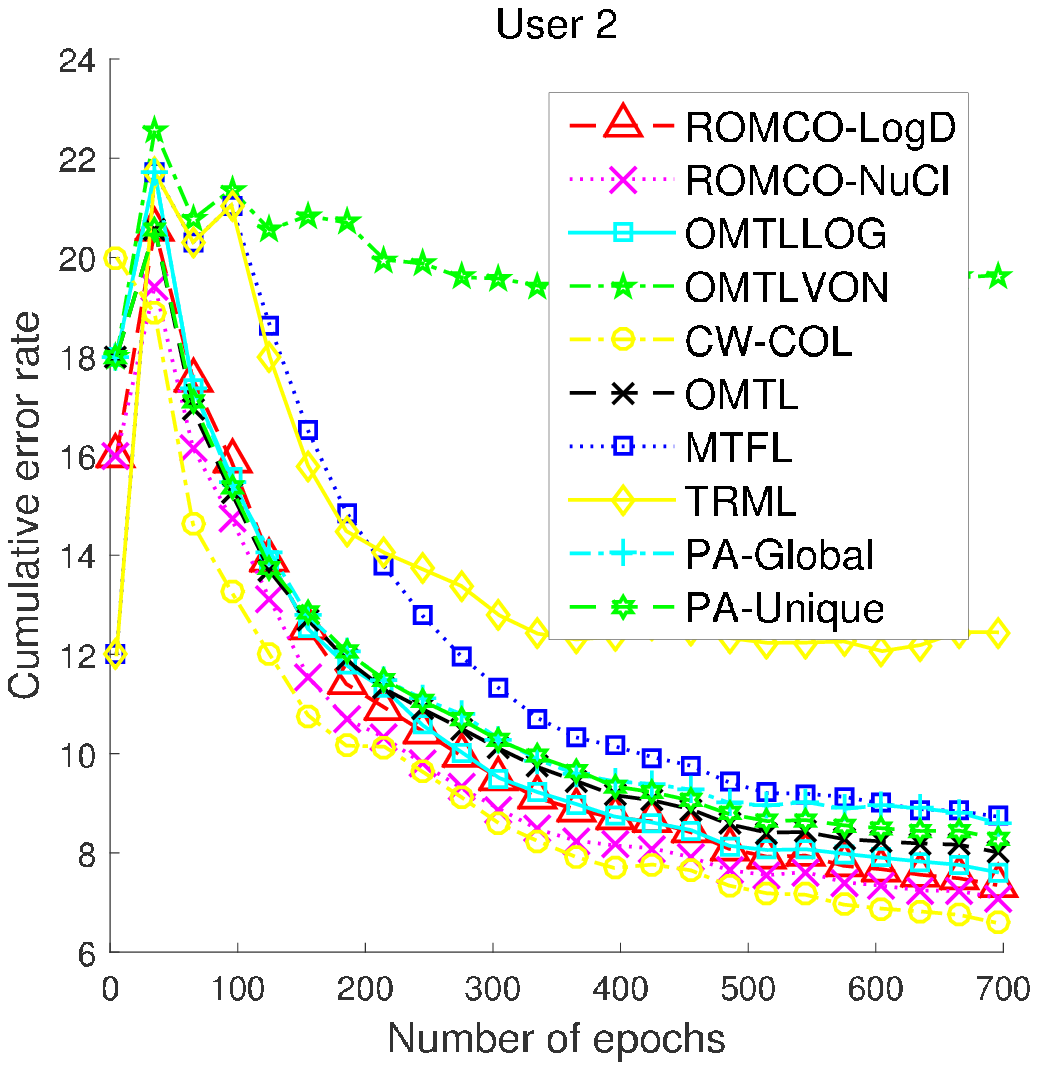}}
\hfil
\subfigure {\includegraphics[width=0.2435\textwidth,height=4.4cm]{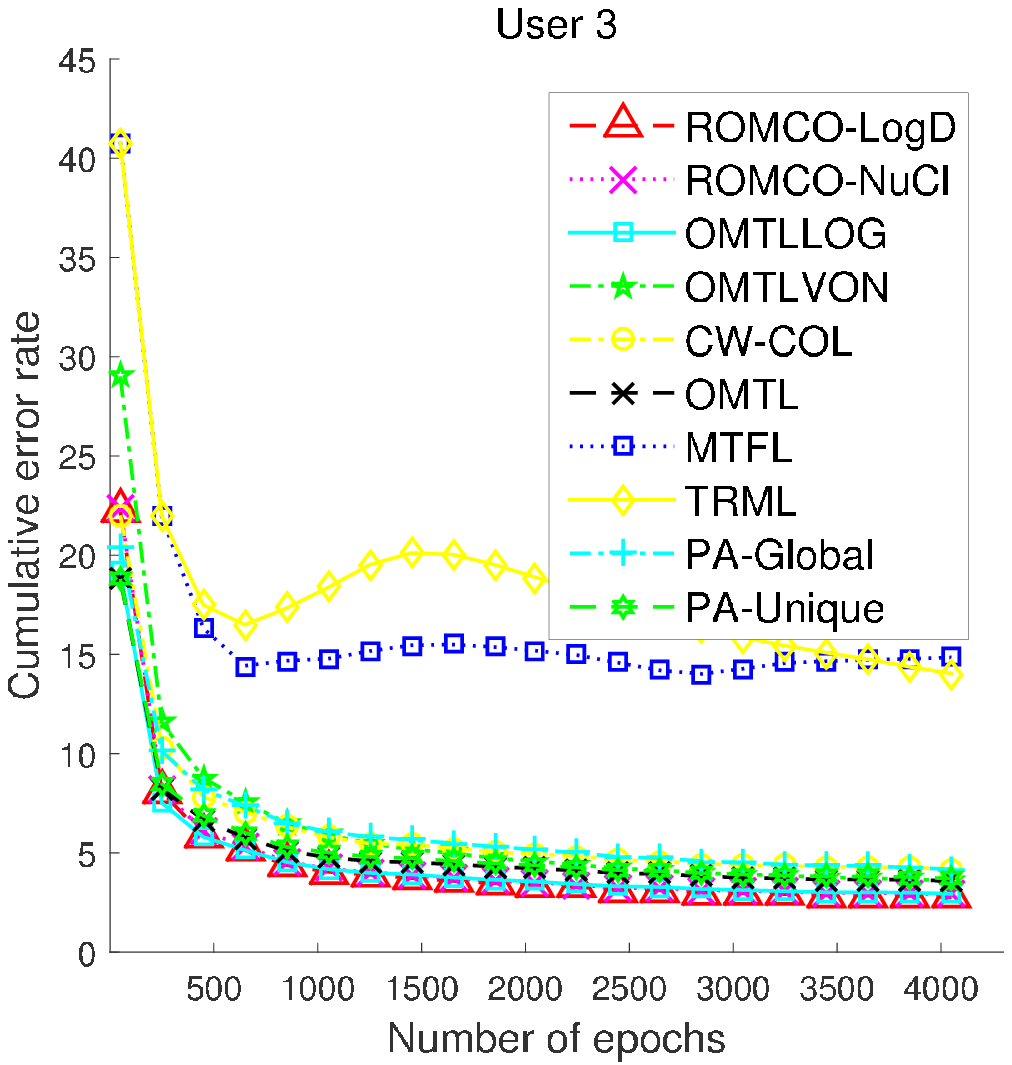}}
\subfigure {\includegraphics[width=0.2435\textwidth,height=4.4cm]{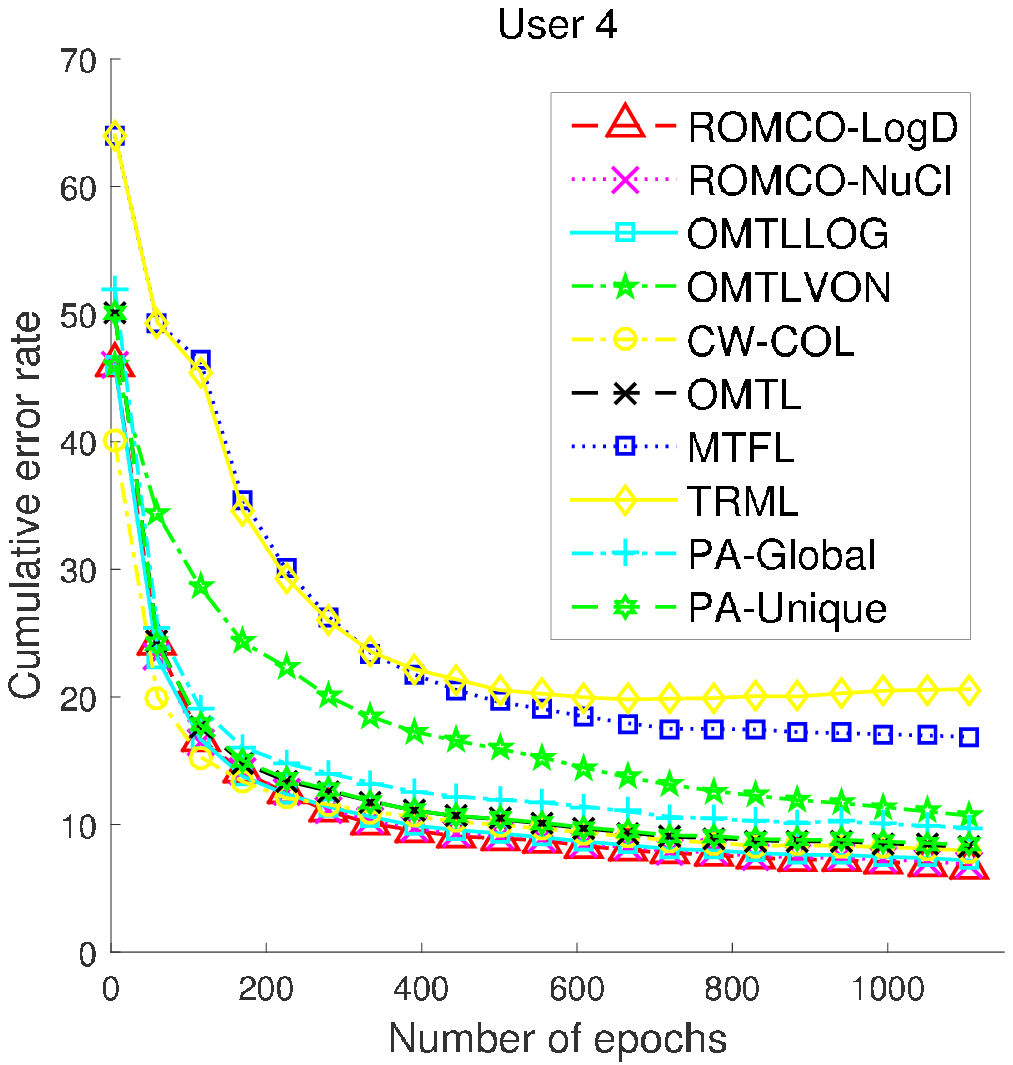}}
\end{figure*}

We applied online multi-task learning to build effective personalized spam filters.
The task is to classify each new incoming email massage into two categories: \emph{legitimate} or \emph{spam}.
We used a dataset hosted by the Internet Content Filtering Group.
The dataset contains 7068 emails collected from mailboxes of four users (denoted by \emph{user1}, \emph{user2}, \emph{user3}, \emph{user4}). Basically, the set of all emails received by a user was not specifically generated for that user.
However, the characteristic of each user's email could be said to match user's interest.
Each mail entry was converted to a word document vector using the TF-IDF (term frequency-inverse document frequency) representation.

Since the email dataset had no time-stamp, each email list was shuffled into a random sequence. The cumulative error rate and
F1-measure results of 10 shuffles are listed in Table \ref{email_comparison_result}. In addition, the cumulative error rate for the four specific users along the learning process is presented in Fig. \ref{EmailSpam-figure}. We also report each algorithm's run-time, i.e., the time consumed by both training and test phase during the complete online learning process in Table \ref{running_time}. From these results, we can make several observations.

First, the proposed ROMCO-NuCl/LogD outperform other online learners in terms of the error rate and F1-measure.
In particular, in accordance to the results from the four specific users, learning tasks collaboratively with both the common and personalized structures consistently beats both the global model and the personalized model.

Second, the performance of the proposed online multi-task learning methods are better than that of the two batch learning algorithms (MTFL and TRML). It should be noted that compared to online learners which update models based only on the current
samples, batch learning methods have the advantage of keeping a substantial amount of recent training samples, at the cost
of storage space and higher complexity. In fact, the proposed ROMCO-NuCL/LogD are more efficient than the batch incremental methods,
e.g., it could be more than 100 times faster than batch MTFL in large-sized dataset (28.72 secs versus 302.17 secs in EachMovie as shown in Table \ref{running_time}). ROMCO-NuCL/LogD do not store recent training samples. They only use the current training sample and a simple rule to update the model. In contrast, batch learning algorithms need to keep a certain number of recent training samples in memory,
learning to extra burden on storage and complexity. In addition, both MTFL and TRML need to solve an optimization problem in
an iterative manner. For practical applications involving hundreds of millions of users and features, the batch learning algorithms are no longer feasible, while online learners remain highly efficient and scalable.

We also observed that ROMCO-NuCL/LogD are slightly slower than CW-COL, OMTLVON/LOG and OSMTL-e. This is expected as ROMCO-NuCL/LogD have to update two component weight matrices. However, the extra computational cost is worth considering the significant improvement over the two measurements achieved by using the two components.

\vspace{-0.05in}
\subsection{MHC-I Binding Prediction}

\begin{table}[t]
\centering
\caption{Cumulative error rate (\%) and F1-measure (\%) and its Standard Deviation in the parenthesis on the MHC-I Dataset Results over 12 Tasks}
\label{Bio-table}
\scriptsize
\begin{tabular}[2.1\textwidth] {|c|c|c|c|}
\hline
\multirow{2}{*}{Algorithm} & \multirow{2}{*}{Error Rate} & Positive Class  & Negative Class \\
  &  & F1-measure &  F1-measure \\
\hline\hline
MTFL				& 43.84(6.05)  & 51.04(10.12) & 59.12(7.35)   \\ \hline
TRML				& 44.26(5.98)  & 50.50(9.97)  & 58.80(7.40)   \\ \hline
PA-Global           & 44.70(2.68)  & 45.44(9.86)  & 61.28(3.17)   \\ \hline
PA-Unqiue           & 41.62(3.95)  & 51.08(10.23) & 63.02(2.89)   \\ \hline
CW-COL			    & 41.32(4.46)  & 50.89(10.70) & 63.59(3.08)	 \\ \hline 
OMTL				& 41.56(3.97)  & 51.13(10.25) & 63.08(2.89)   \\ \hline
OSMTL-e             & 42.78(0.65)  & 50.48(0.43)  & 61.59(0.96)   \\ \hline
OMTLVON             & 38.13(5.03)  & 54.73(10.44) & 66.39(4.02)   \\ \hline
OMTLLOG             & 38.08(5.16)  & 54.83(10.56) & 66.40(4.13)    \\ \hline \hline
\textbf{ROMCO-NuCl} & 38.09(5.32)  & 55.03(10.53) & 66.34(4.09)   \\ \hline
\textbf{ROMCO-LogD} & \textbf{37.91(4.97)}    & \textbf{55.09(10.10)}   & \textbf{66.55(4.14)}  \\ \hline
\end{tabular}
\end{table}

\begin{figure*}[t]
\centering
\caption{Cumulative error rate on the 12 tasks of MHC-I dataset along the entire online learning process}
\label{Bio-figure}
\subfigure {\includegraphics[width=0.2445\textwidth,height=4.4cm]{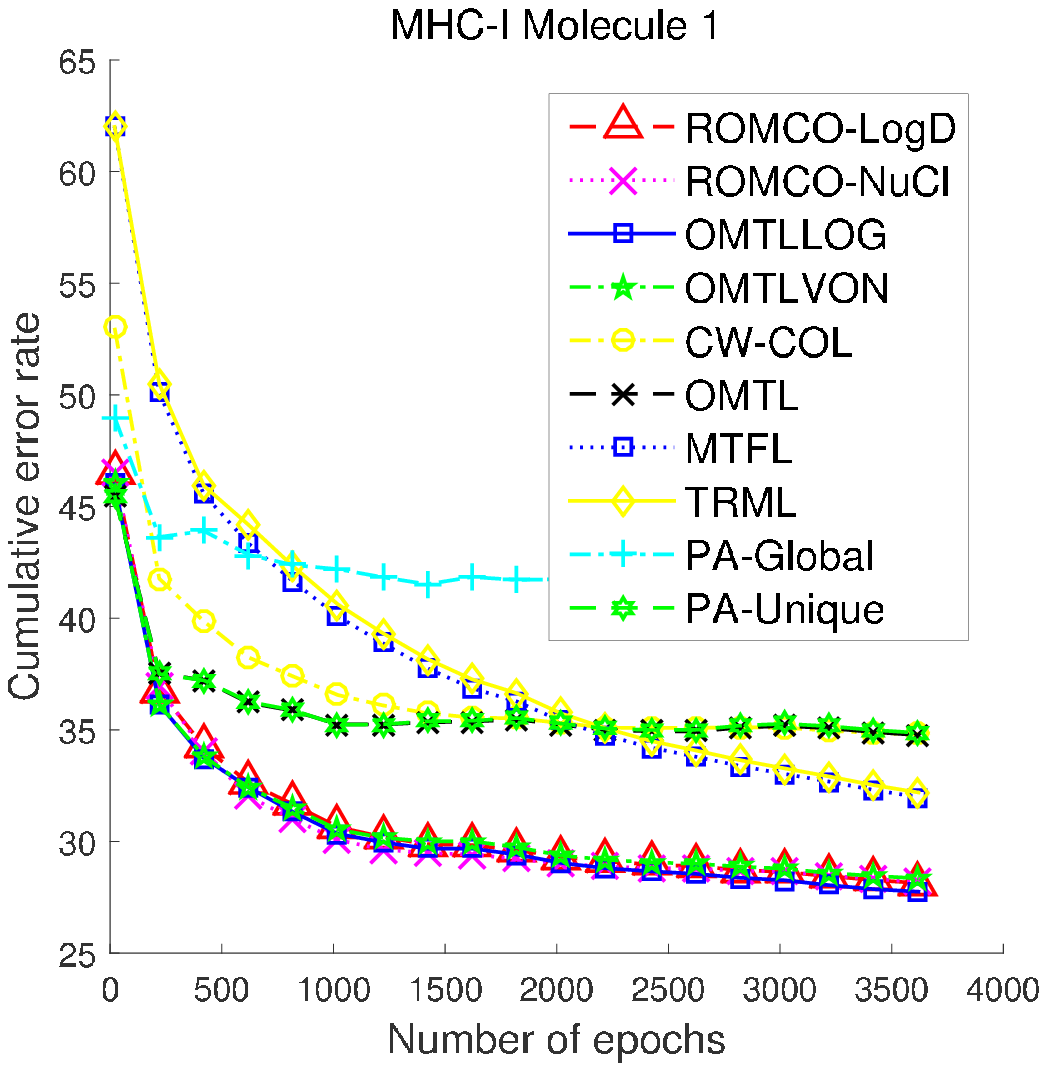}}
\subfigure {\includegraphics[width=0.2445\textwidth,height=4.4cm]{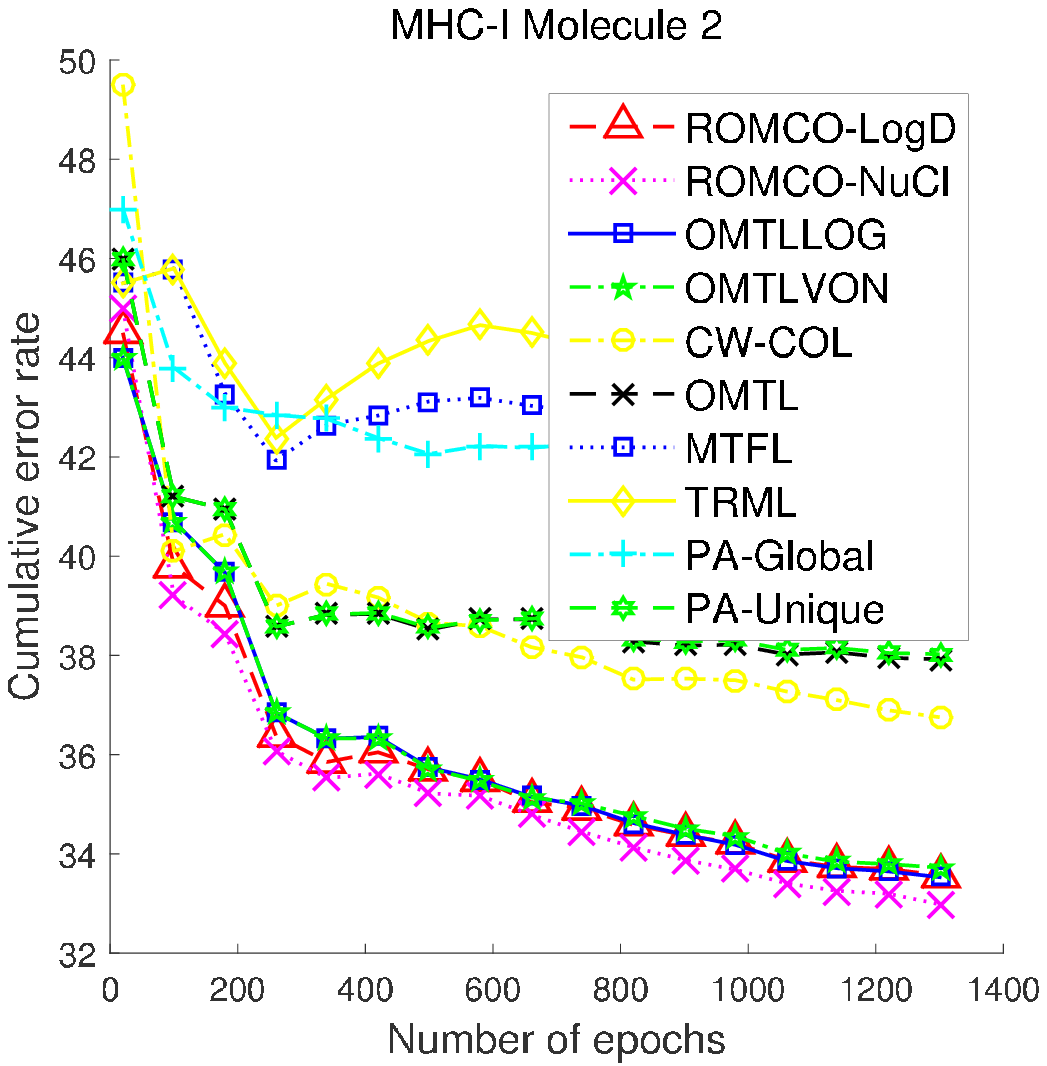}}
\hfil
\subfigure {\includegraphics[width=0.2445\textwidth,height=4.4cm]{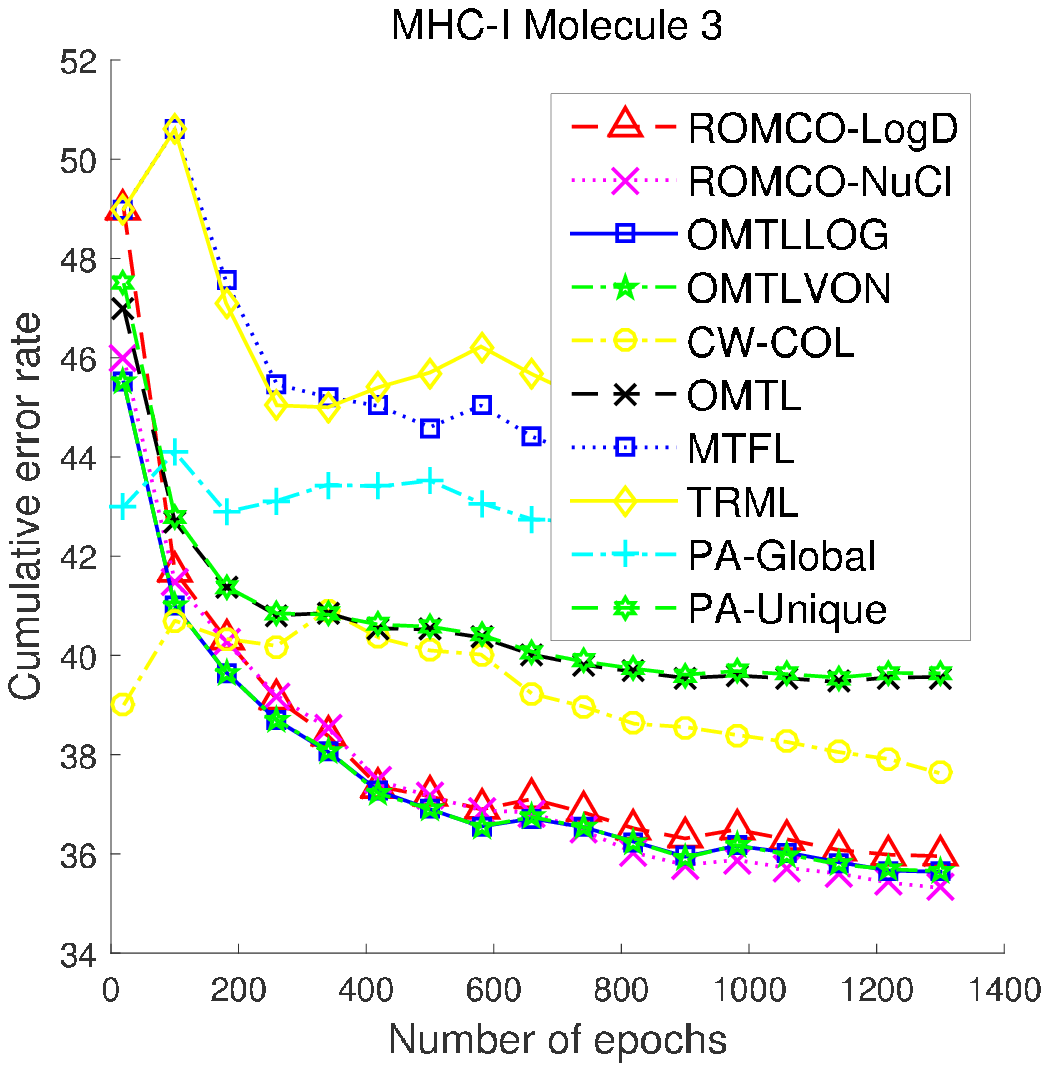}}
\subfigure {\includegraphics[width=0.2445\textwidth,height=4.4cm]{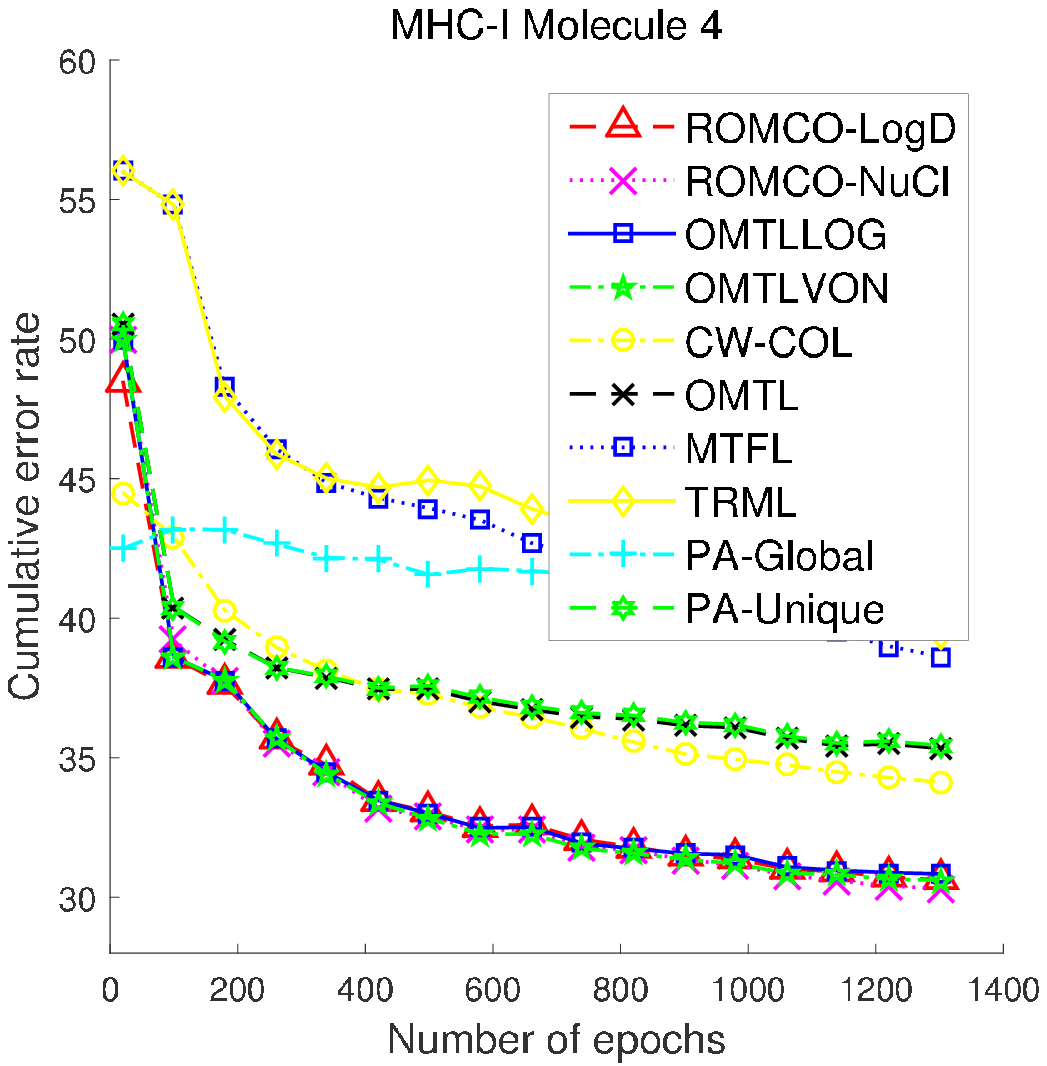}}
\hfil
\subfigure {\includegraphics[width=0.2445\textwidth,height=4.4cm]{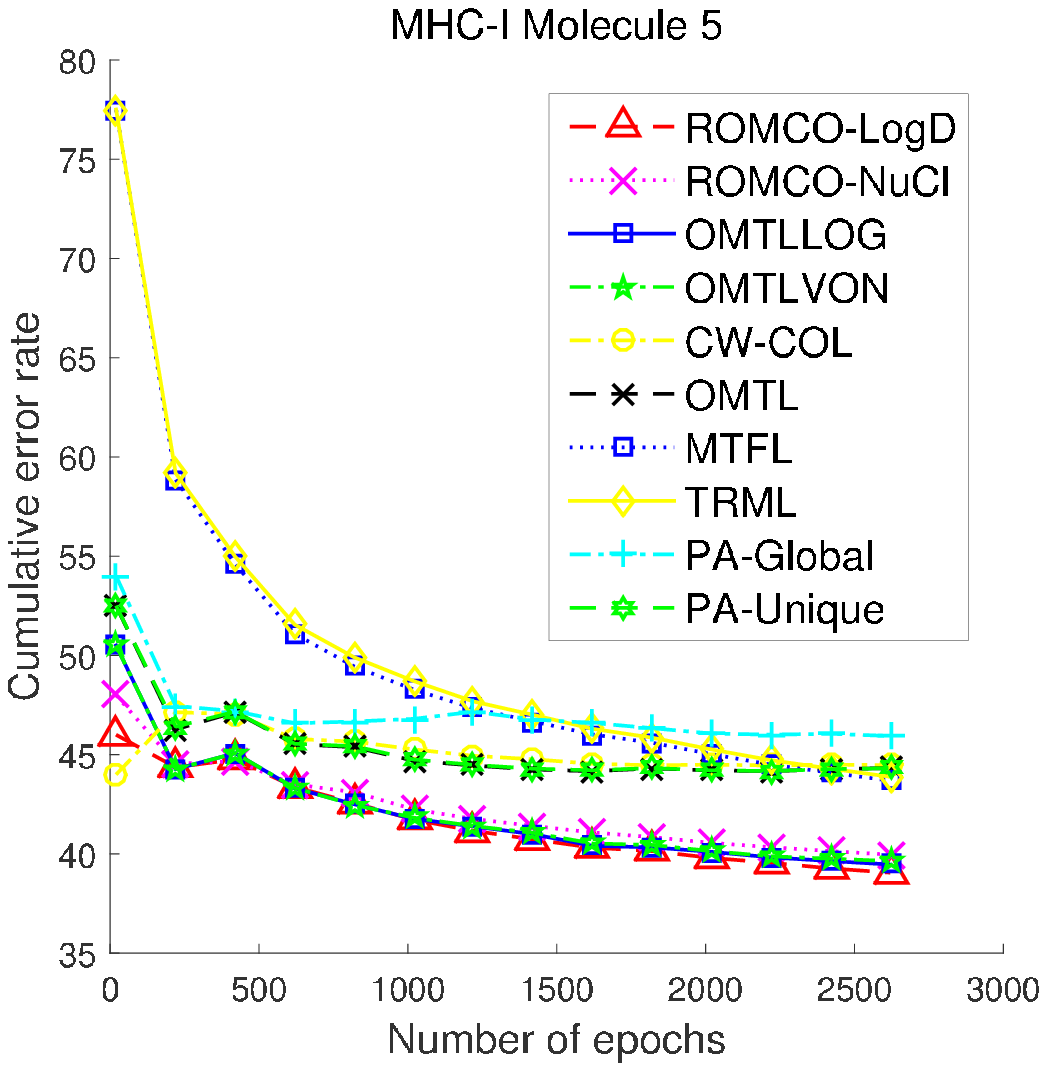}}
\subfigure {\includegraphics[width=0.2445\textwidth,height=4.4cm]{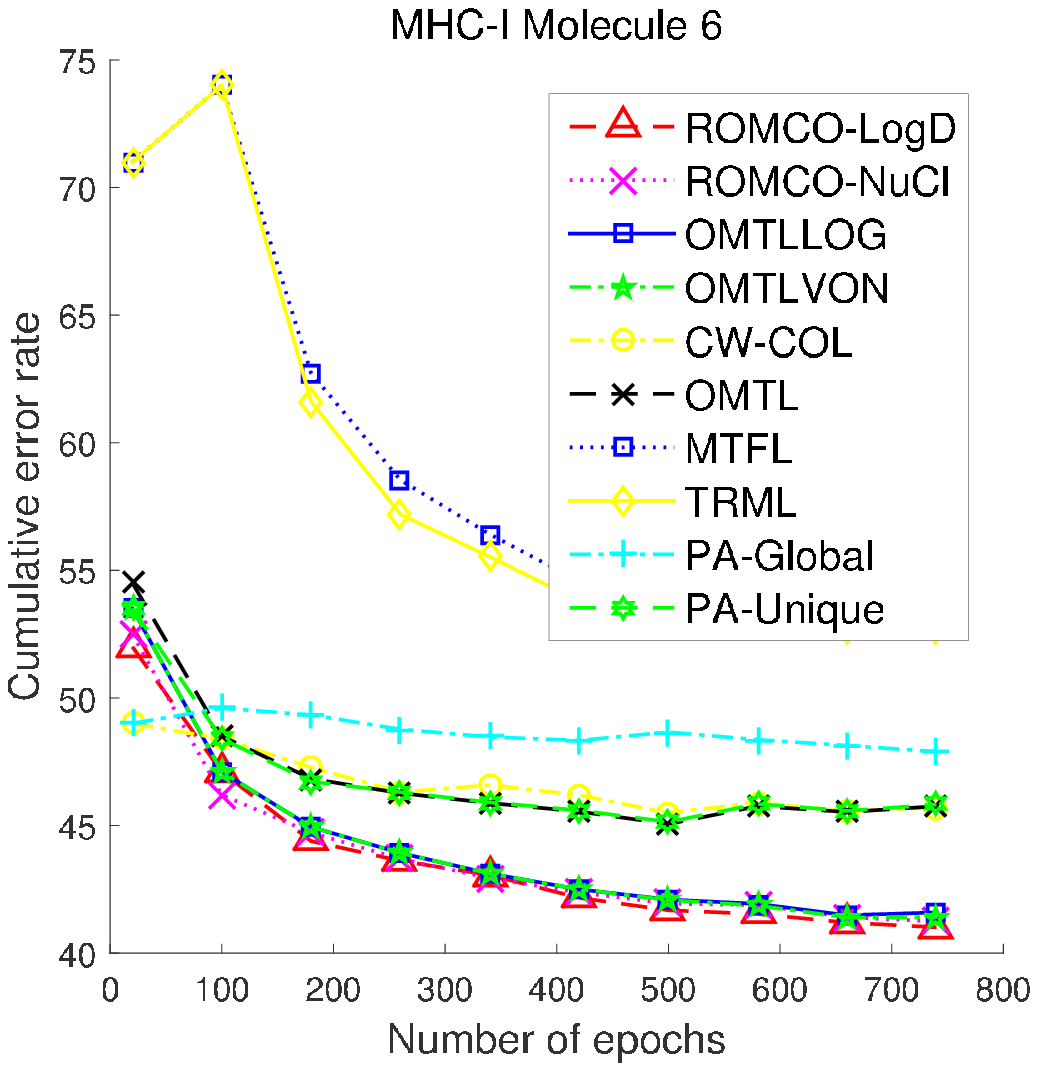}}
\hfil
\subfigure {\includegraphics[width=0.2445\textwidth,height=4.4cm]{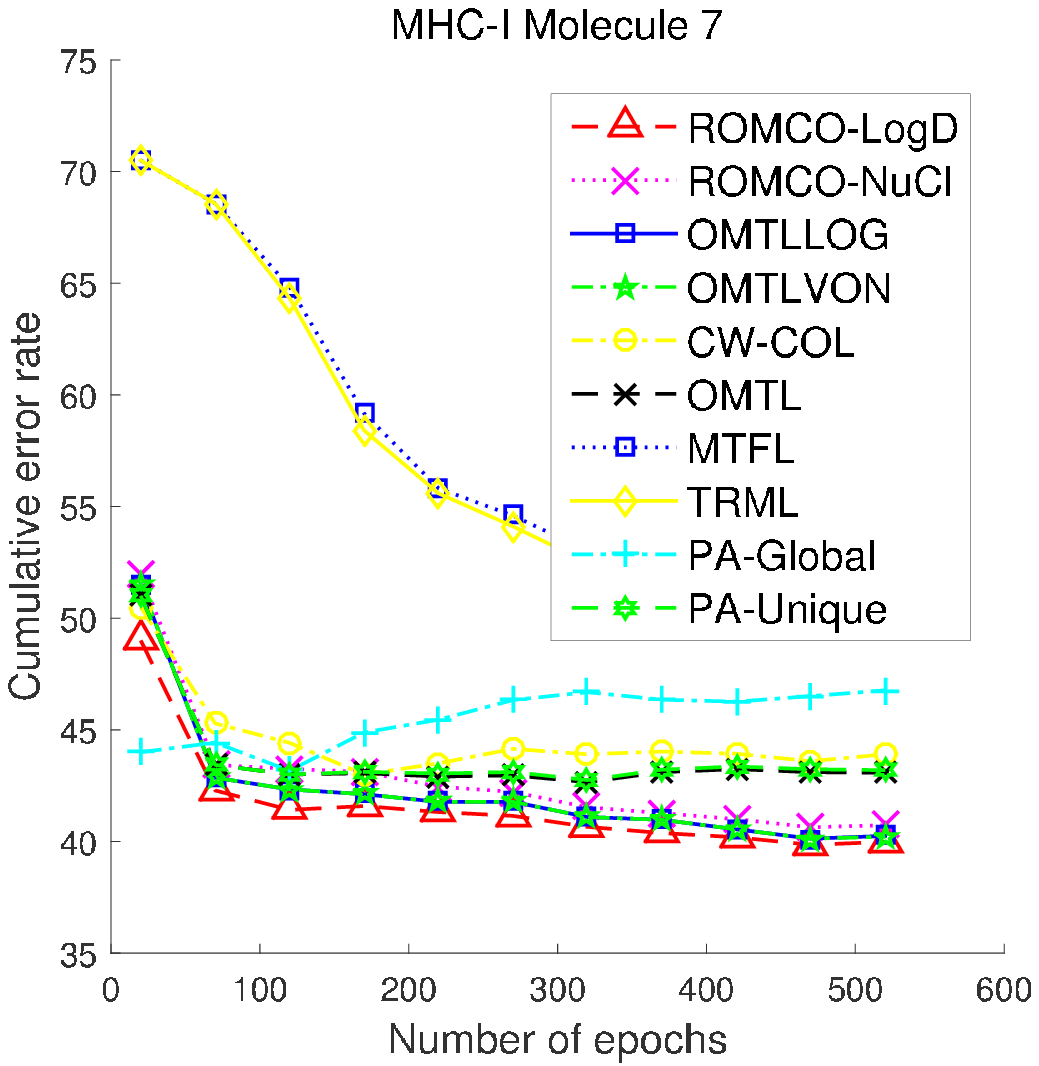}}
\subfigure {\includegraphics[width=0.2445\textwidth,height=4.4cm]{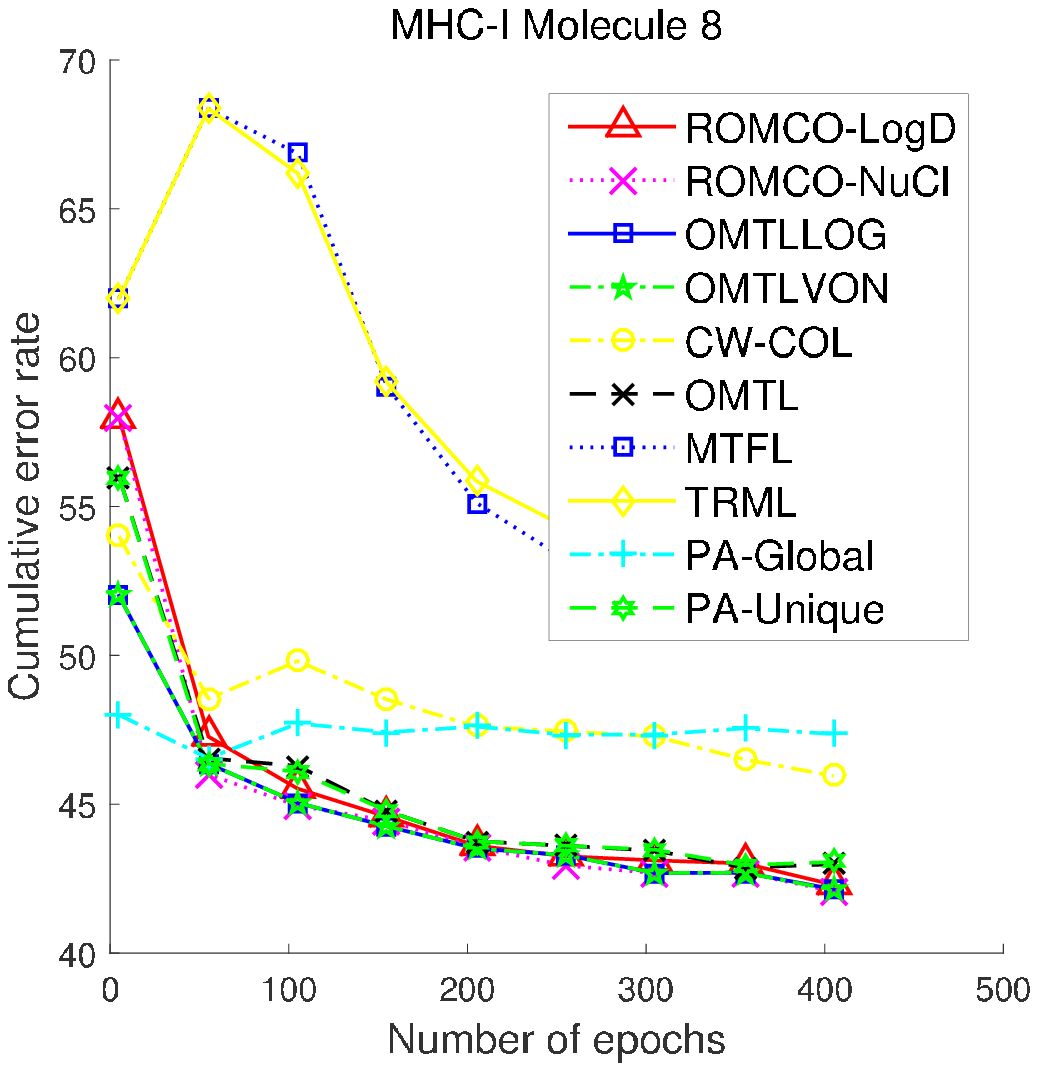}}
\hfil
\subfigure {\includegraphics[width=0.2445\textwidth,height=4.4cm]{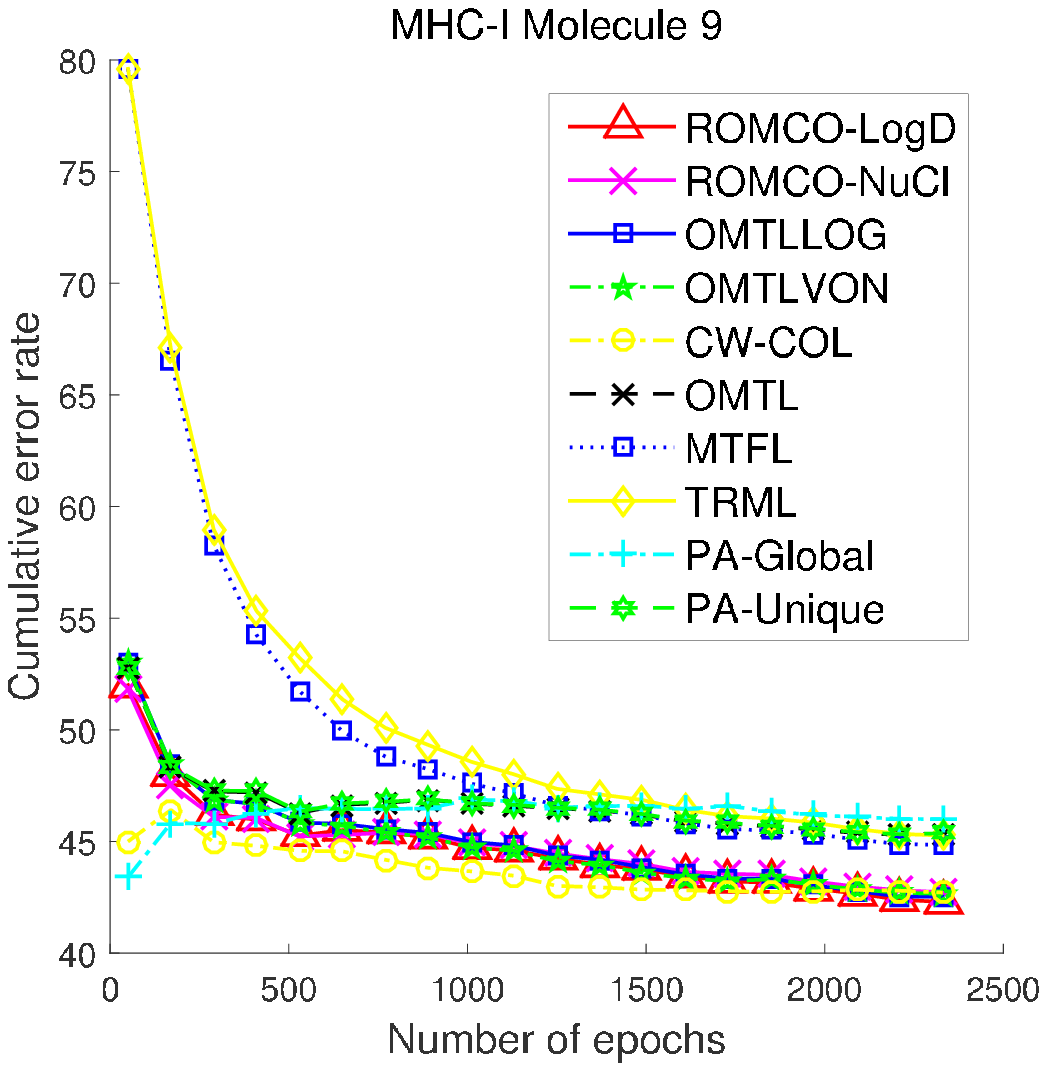}}
\subfigure {\includegraphics[width=0.2445\textwidth,height=4.4cm]{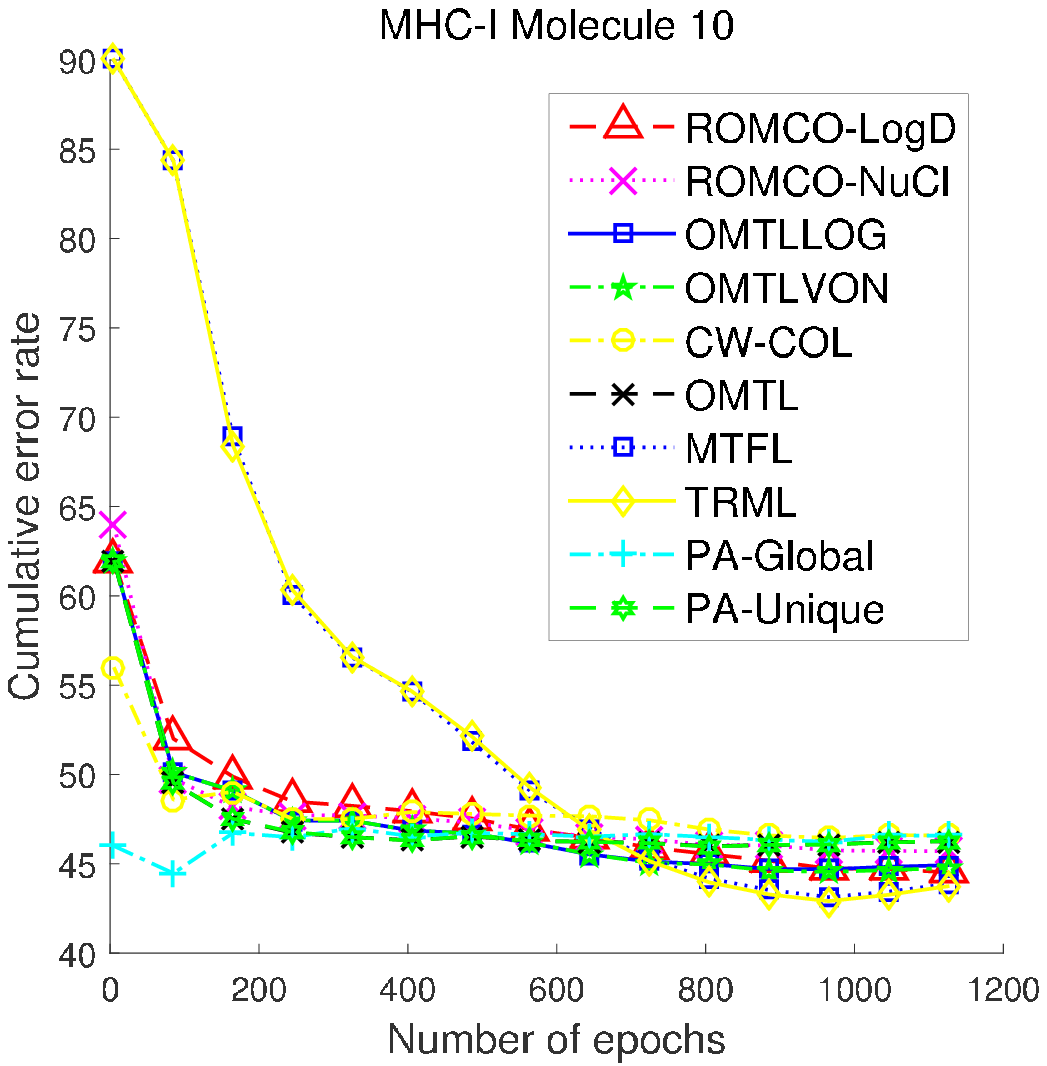}}
\hfil
\subfigure {\includegraphics[width=0.2445\textwidth,height=4.4cm]{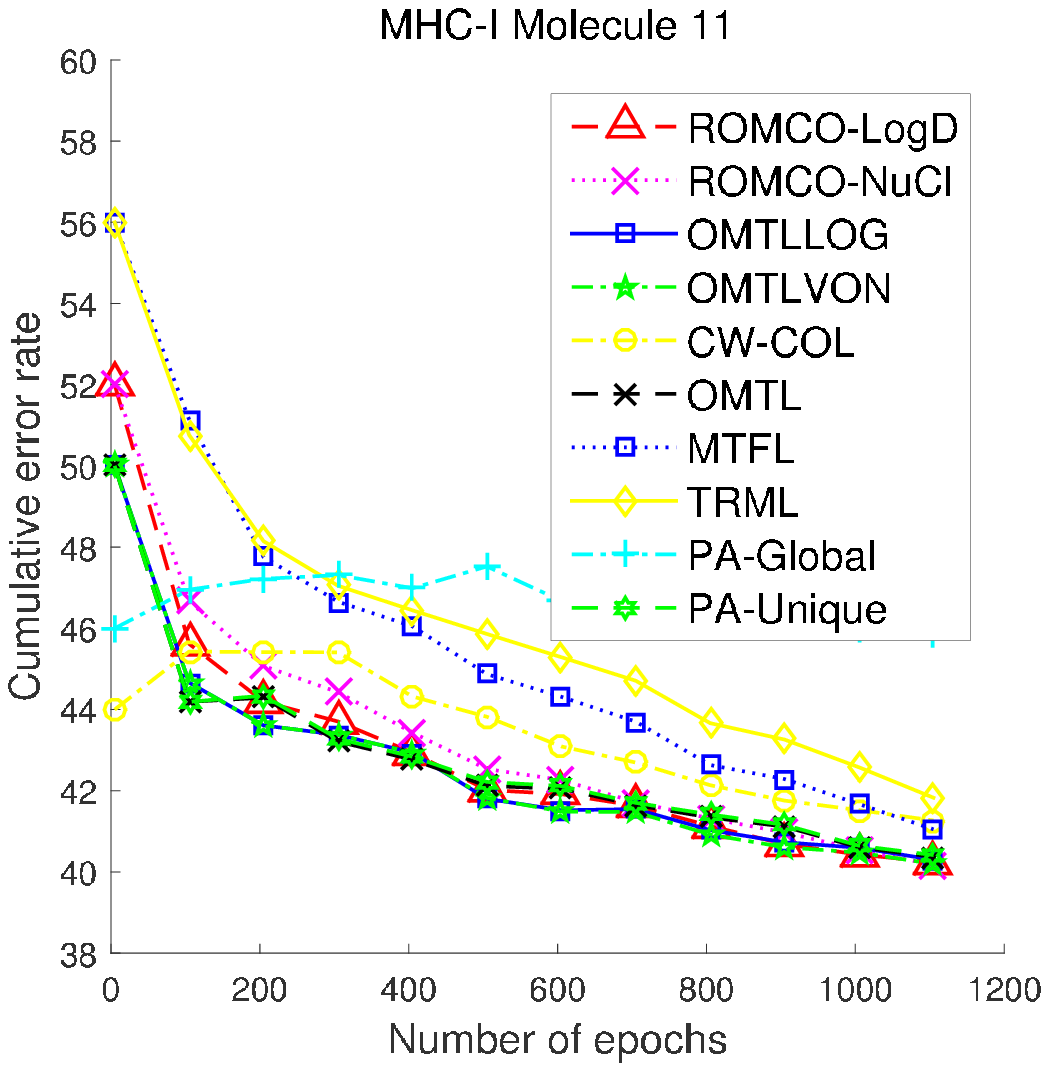}}
\subfigure {\includegraphics[width=0.2445\textwidth,height=4.4cm]{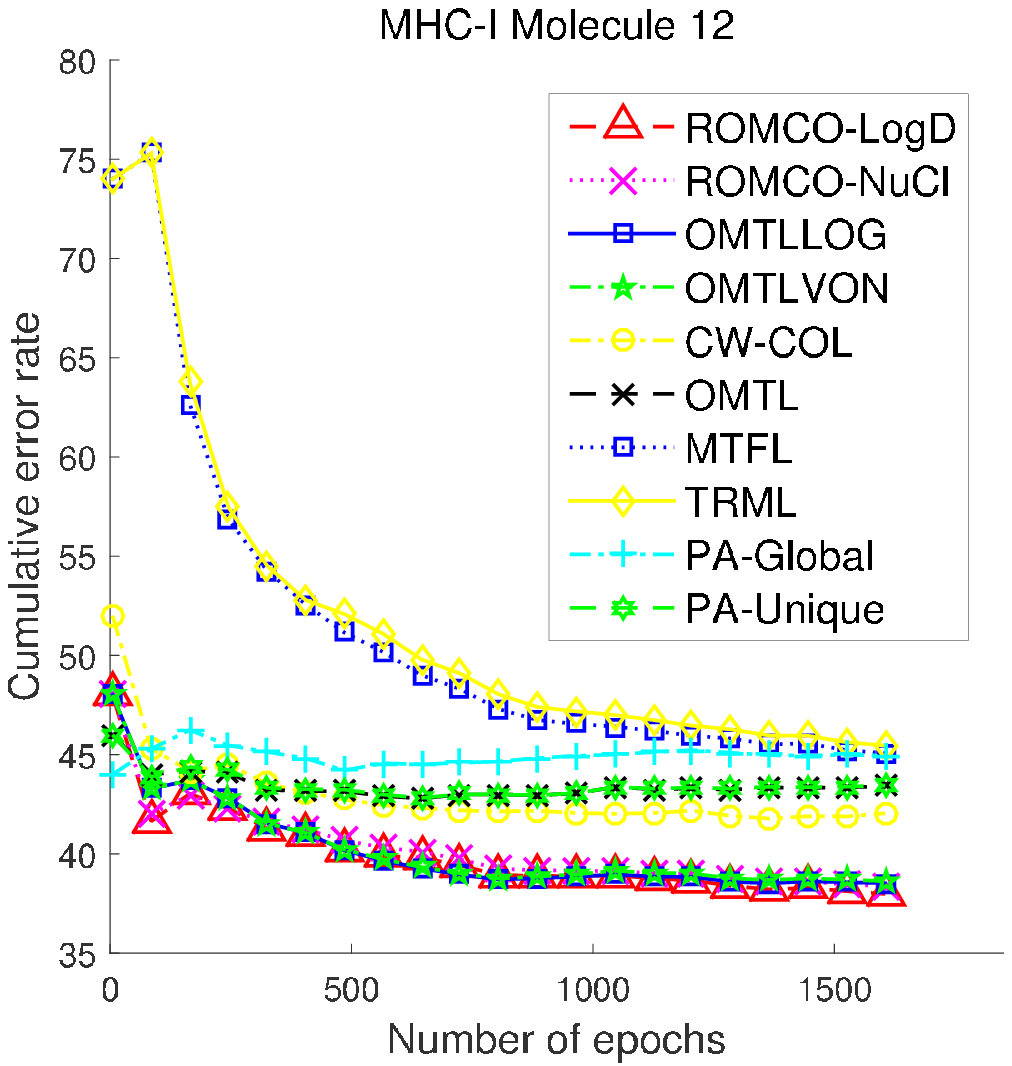}}
\end{figure*}

Computational methods have been widely used in bioinformatics to build models to infer properties from biological data~\cite{yang2014ensemble,yang2014ldsplit}. In this experiment, we evaluated several methods to predict peptide binding to human MHC (major histocom-patibility complex) class I molecules. It is known that peptide binding to human MHC-I molecules plays a crucial role in the immune system. The prediction of such binding has valuable application in vaccine designs, the diagnosis and treatment of cancer, etc. Recent
work has demonstrated that there exists common information between related molecules (alleles) and such information can be
leveraged to improve the peptide MHC-I binding prediction.

We used a binary-labeled MHC-I dataset.
The data consists of 18664 peptide sequences for 12 human MHC-I molecules.
Each peptide sequence was converted to a 400 dimensional feature vector following~\cite{LiCHLJ11}.
The goal is to determine whether a peptide sequence (instance) binds to a MHC-I molecule (task) or not, i.e., \emph{binder} or \emph{non-binder}.

We reported the average cumulative error rate and F1-measure of 12 tasks in Table \ref{Bio-table}. To make a clear comparison
between the proposed ROMCO-NuCL/LogD and baselines, we showed the variation of their cumulative error rate along the entire online learning process averaged over the 10 runs in Fig. \ref{Bio-figure}.

From these results, we first observed that the permutations of the dataset have little influence on the performance of each method,
as indicated by the small standard deviation values in Table \ref{Bio-table}. Note that the majority of the dataset belongs to the negative class, thus predicting more
examples as the majority class decreases the overall error rate, but also degrades the accuracy of the minority positive class.
The consistently good performance achieved by the proposed ROMCO-NuCL/LogD in terms of the error rate and F1-measures of both classes further demonstrates effectiveness of our algorithms over imbalanced datasets. Moreover, among the six online models, learning related tasks jointly still achieves better performance than learning the tasks individually, as shown by the improvement of ROMCO and OMTL models over the PA-Unique model.

\begin{figure}[t]
\centering
\caption{Average error rate and F1-measure on the EachMovie dataset over 30 tasks along the entire online learning process}\label{EachMovie_figure}
\subfigure {\includegraphics[width=0.24\textwidth,height=4.5cm]{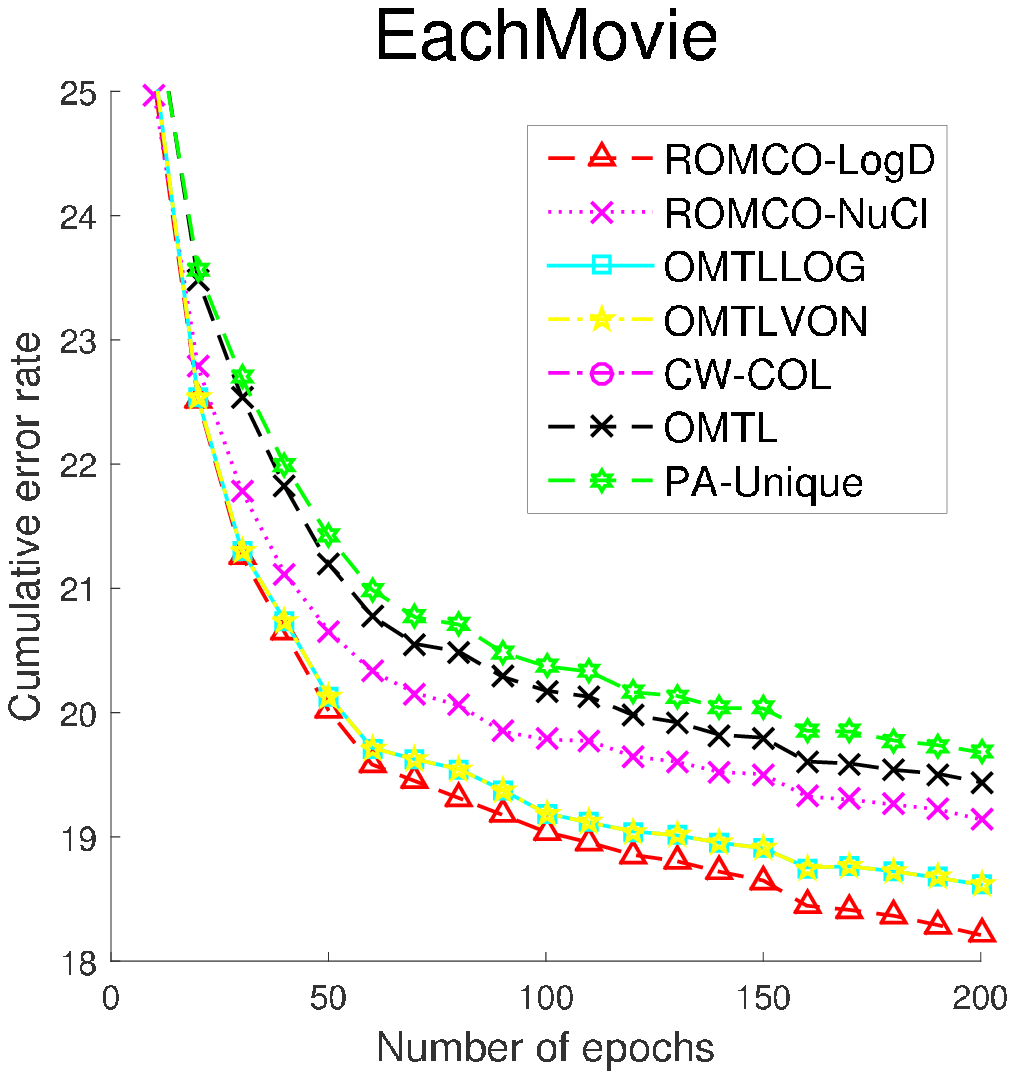}}
\subfigure {\includegraphics[width=0.24\textwidth,height=4.5cm]{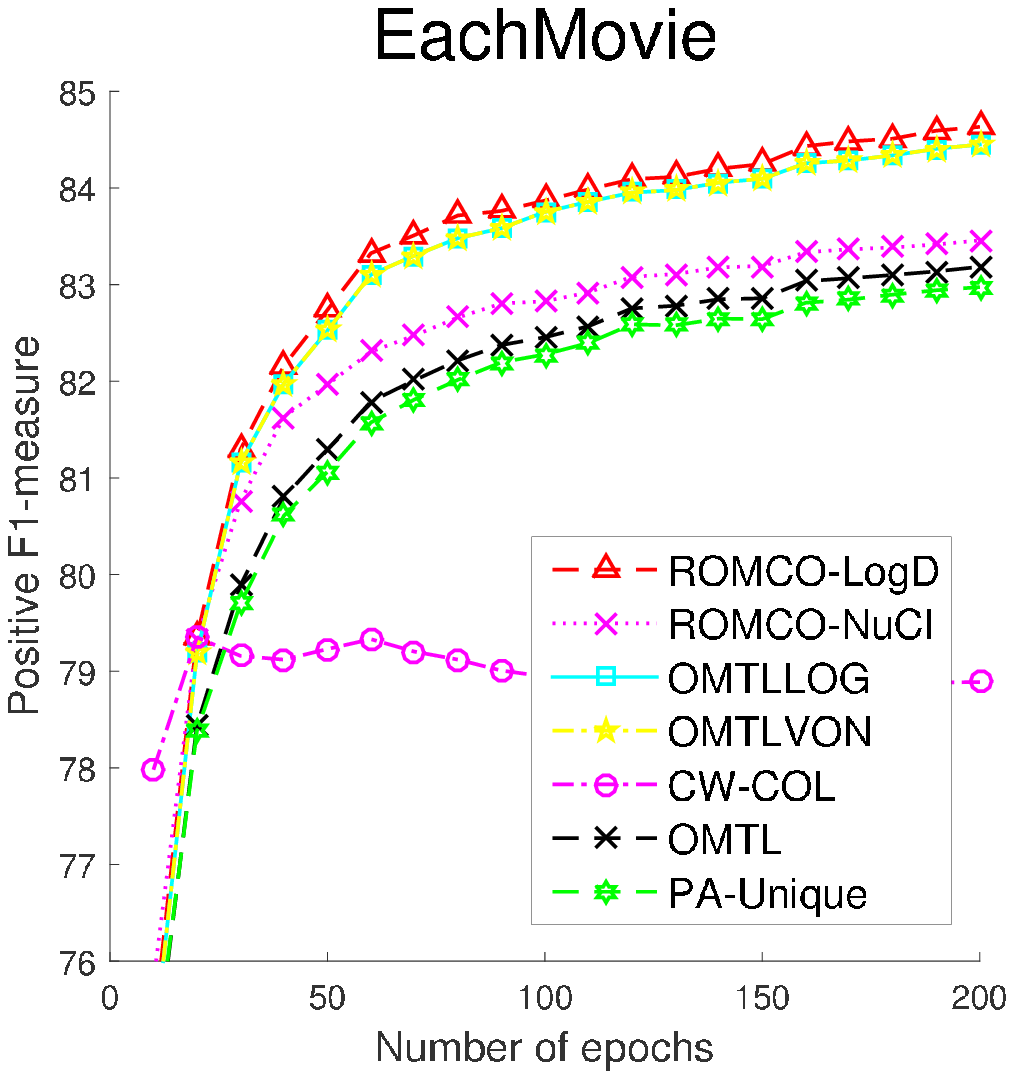}}
\end{figure}

\begin{table}[t]
\centering
\caption{Average Cumulative Error rate (\%) and F1-measure(\%) with their standard deviation in the parenthesis over 30 tasks of EachMovie Dataset Results}
\label{EachMovie-table}
\scriptsize
\begin{tabular}[2.1\textwidth] {|c|c|c|c|}
\hline
\multirow{2}{*}{Algorithm} & \multirow{2}{*}{Error Rate} & Positive Class  & Negative Class \\
  &  & F1-measure &  F1-measure \\
\hline\hline
MTFL				& 27.51(12.25) & 79.18(12.87) &  36.06(14.85) \\ \hline
TRML				& 26.58(11.82) & 79.89(12.49) &  37.64(15.05) \\ \hline
PA-Global           & 31.80(5.87) & 74.43(8.61)   &  47.96(14.47) \\ \hline
PA-Unqiue           & 19.68(7.39) & 82.97(9.35)   &  57.80(21.05) \\ \hline
CW-COL			    & 25.45(6.96) & 78.89(9.30)   &	 53.95(16.71) \\ \hline 
OMTL				& 19.44(7.28) & 83.18(9.29)   &  57.77(21.39)  \\ \hline
OSMTL-e             & 20.73(7.15) & 82.31(9.04)   &  \textbf{58.76(18.71)}  \\ \hline
OMTLVON             & 18.61(7.29) & 84.45(8.64)   &  55.92(23.94) \\ \hline
OMTLLOG             & 18.61(7.29) & 84.45(8.64)   &  55.92(23.94)  \\ \hline \hline
\textbf{ROMCO-NuCl} & 19.14(7.20) & 83.46(9.25)   &  58.27(21.17) \\ \hline
\textbf{ROMCO-LogD} & \textbf{18.21(6.71)}    & \textbf{84.63(8.41)}   & 55.53(25.02)  \\ \hline
\end{tabular}
\end{table}

\vspace{-0.05in}
\subsection{Movie Recommender System}

\begin{figure*}
\centering
\caption{Sensitivity analysis on the effect of the parameter $\lambda_1$ and $\lambda_2$ in terms of the error rate}\label{Sensitivity_Figure}
\subfigure {\includegraphics[width=0.2445\textwidth,height=4.6cm]{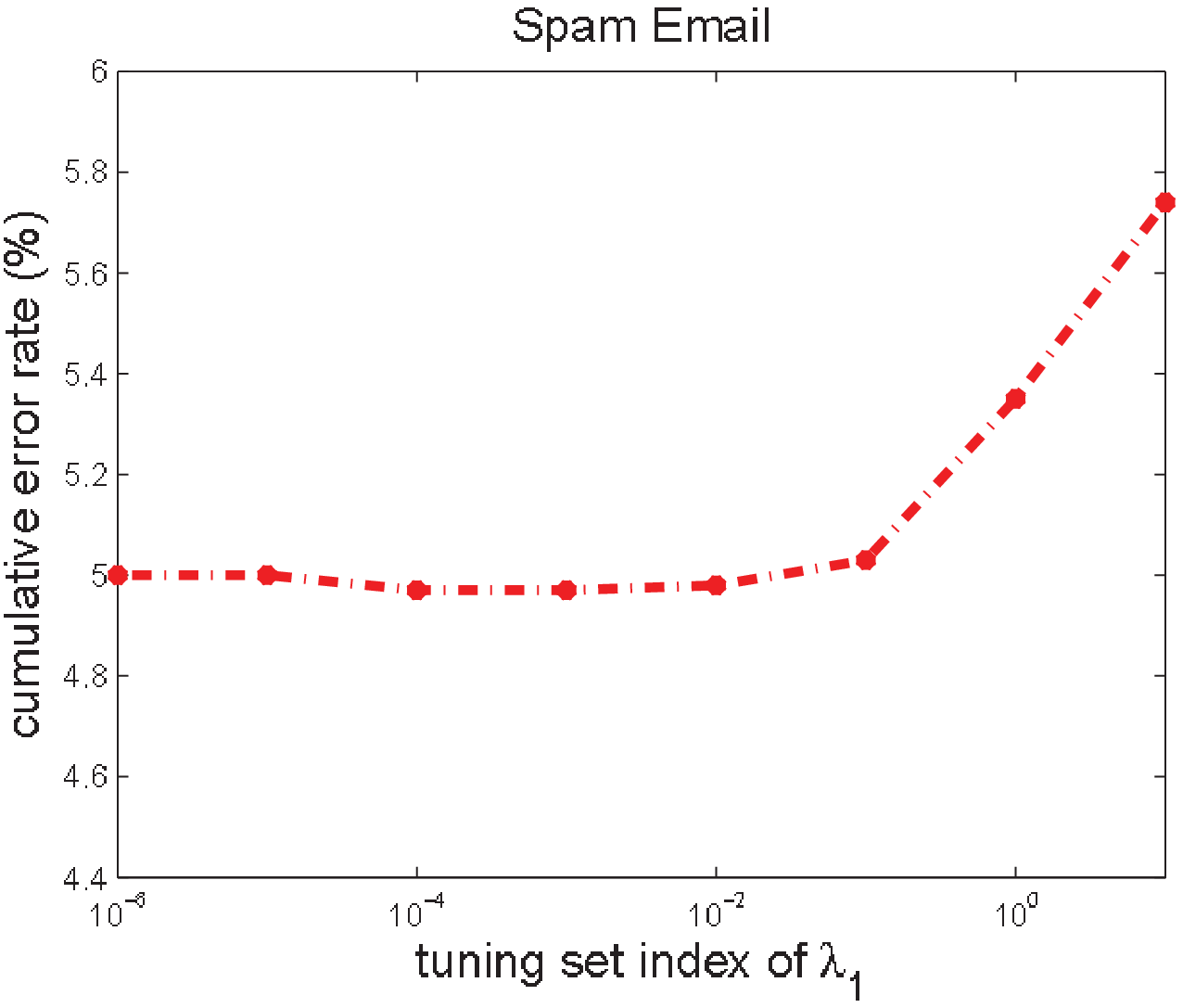}}
\subfigure {\includegraphics[width=0.2445\textwidth,height=4.6cm]{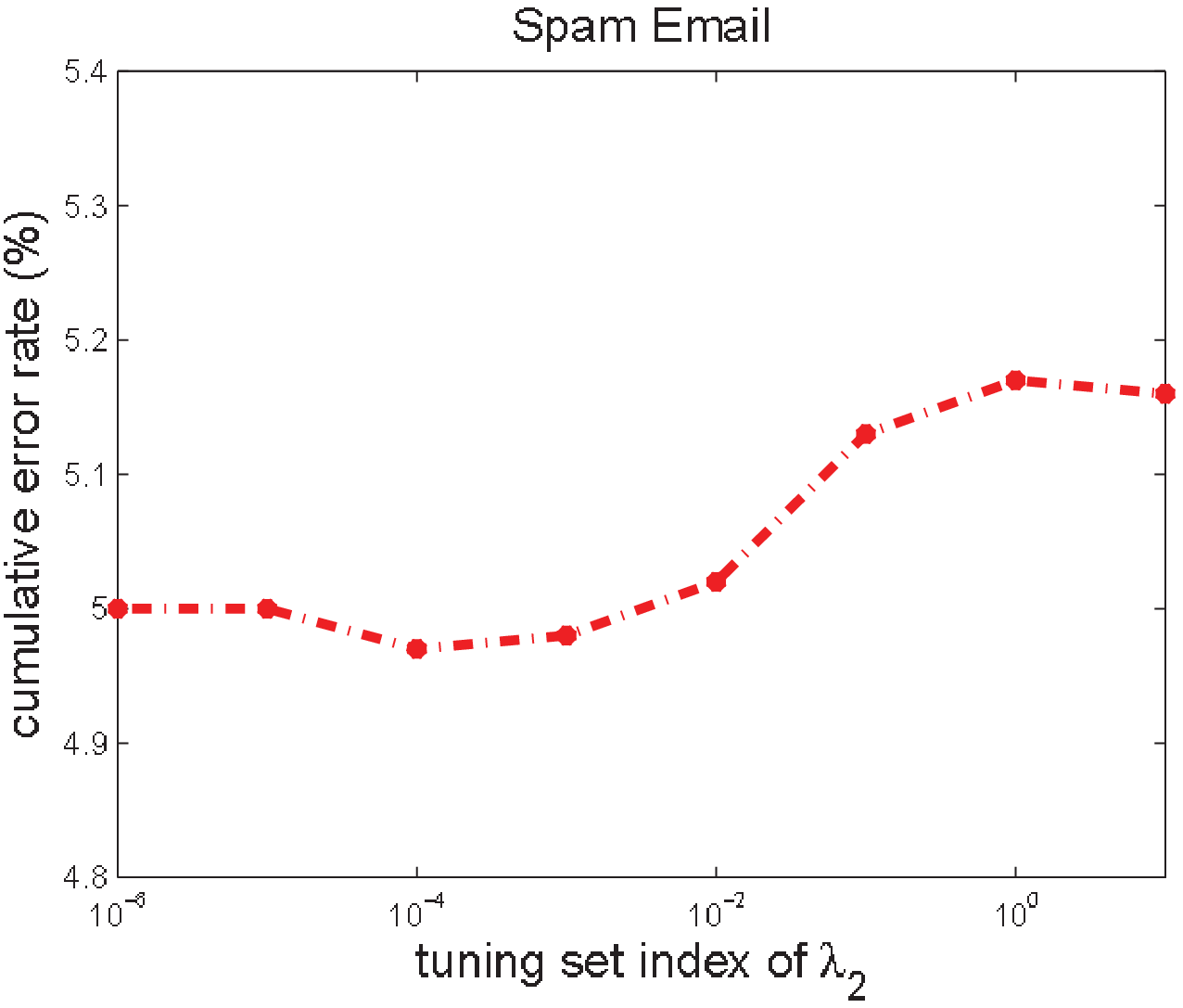}}
\subfigure {\includegraphics[width=0.2445\textwidth,height=4.6cm]{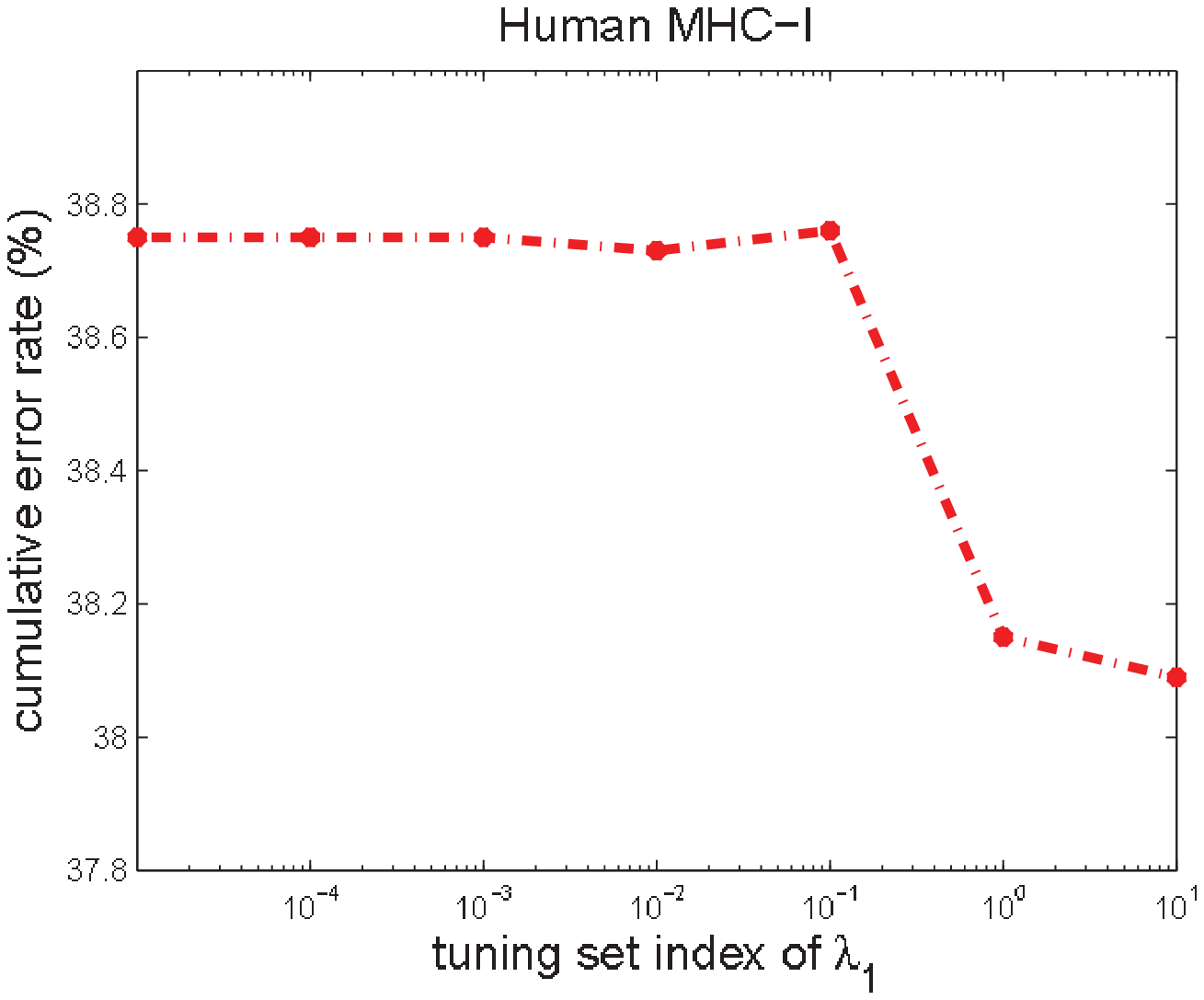}}
\subfigure {\includegraphics[width=0.2445\textwidth,height=4.6cm]{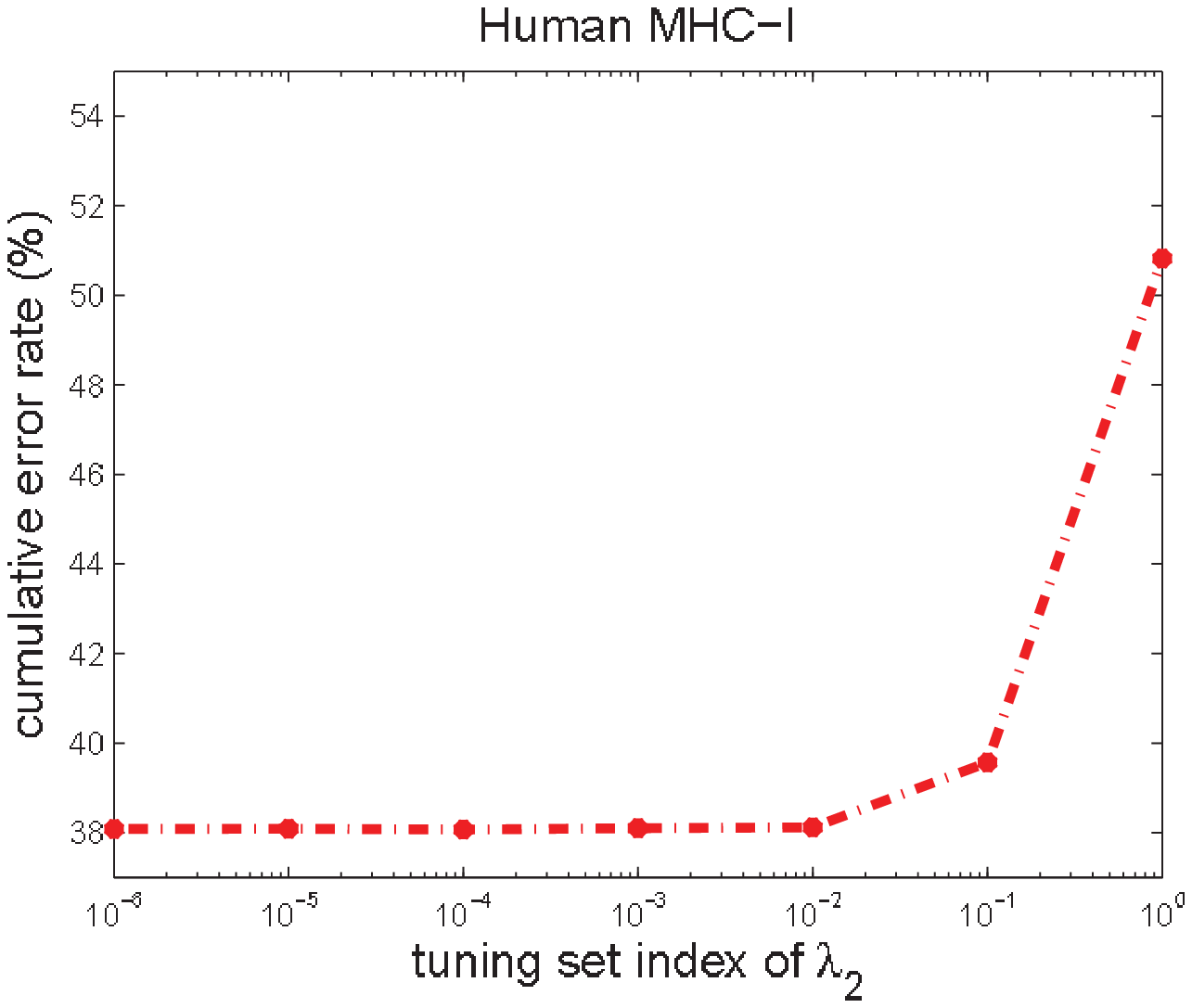}}
\end{figure*}

In recent years, recommender systems have achieved great success in many real-world applications.
The goal is to predict users' preferences on targeted products, i.e., given partially observed user-movie rating entries
in the underlying (ratings) matrix, we would like to infer their preference for unrated movies.

We used a dataset hosted by the DEC System Research Center that collected the EachMovie recommendation data for 18 months.
During that time, 72916 users entered a total of 2811983 numeric ratings for 1628 different movies.
We randomly selected 30 users (tasks) who viewed exactly 200 movies with their rating as the target classes.
Given each of the 30 users, we then randomly selected 1783 users who viewed the same 200 movies and used their ratings as the features of the movies.
The six possible ratings (i.e., $\{1,\ldots,6\}$) were converted into binary classes (i.e., \emph{like} or \emph{dislike}) based on the rating order.

Table~\ref{EachMovie-table} shows the comparison results in terms of the average cumulative error rate and F1-measure.
Fig.~\ref{EachMovie_figure} depicts the detailed the cumulative error rate along the entire online learning process over the averaged 30 tasks of the EachMovie dataset.
From these results, we can make several conclusions.

First, it can be seen that the proposed ROMCO-NuCl/LogD outperform other baselines: ROMCO-NuCl/LogD always provide smaller error rates and higher F1-measures compared to other baselines.
It shows that our algorithms can maintain a high quality of prediction accuracy.
We believe that the promising result is generally due to two reasons: First, the \emph{personalized} and \emph{correlative} patterns are effective to discover the personalized tasks and task relativeness, and these patterns are successfully captured in three real-world datasets.
Second, once an error occurs in at least one task, ROMCO-NuCl/LogD would update the entire task matrix.
That would benefit other related tasks with few learning instances since the shared subspaces would be updated accordingly.

Next, we observed that ROMCO-LogD is better than ROMCO-NuCl in Fig. ~\ref{EachMovie_figure} in terms of the error rate and F1-measure. This is expected because compared to the nuclear norm, ROMCO-LogD is able to achieve better rank approximation with the log-determinant function, i.e., it reduces the contribution of large singular values while approximating the small singular values into zeros.

\subsection{Effect of the Regularization Parameters}
\vspace{-0.05in}

We used Spam Email and Human MHC-I datasets as the cases for parameter sensitivity analysis.
In the Spam Email dataset, by fixing $\lambda_2=0.0001$ as well as varying the value of $\lambda_1$ in the tuning set, i.e., $[10^{-6},\ldots,10^0]$, we studied how the parameter $\lambda_1$ affects the classification performance of ; by fixing $\lambda_1=0.0001$ as well as varying the value of $\lambda_2$ in tuning set of $[10^{-7},\ldots,10]$, we study how the parameter affects the performance of ROMCO-NuCl/LogD.
Similarly, in the Human MHC-I dataset, we studied the pair of $(\lambda_1,\lambda_2)$ by fixing $\lambda_2=0.0001$ with the tuning set of $\lambda_1$ $[10^{-5},\ldots,10^0]$ and by fixing $\lambda_1=1$ with tuning set of $\lambda_2$ $[10^{-6},\ldots,10^0]$.
In Fig.~\ref{Sensitivity_Figure}, we show the classification performance of ROMCO in terms of the error rate for each pair of $(\lambda_1,\lambda_2)$.
From Fig.~\ref{Sensitivity_Figure}, we observed that the performance is worse with an increment of either $\lambda_1$ or $\lambda_2$ over the Spam Email dataset.
It indicates a weak relativeness among the tasks and many personalized tasks existing in the Email dataset.
In Human MHC-I, the poor performance is triggered by a small value of $\lambda_1$ or a large value of $\lambda_2$.
Compared with the Email data, MHC-I contains fewer personalized tasks, meanwhile most tasks are closely related and well represented by a low-dimensional subspace.

\vspace{-0.1in}
\section{Conclusion}

We proposed an online MTL method that can identify sparse personalized patterns for outlier tasks, meanwhile capture a shared low-rank subspace for correlative tasks. As an online technique, the algorithm can achieve a low prediction error rate via leveraging previously learned knowledge. As a multi-task approach, it can balance this trade-off by learning all the tasks jointly. In addition, we proposed a log-determinant function to approximate the rank of the matrix, which, in turn, achieves better performance than the one with the nuclear norm. We show that it is able to achieve a sub-linear regret bound with respect to the best linear model in hindsight, which can be regarded as a theoretical support for the proposed algorithm. Meanwhile, the empirical results demonstrate that our algorithms outperform other state-of-the-art techniques on three real-world applications. In future work, online active learning could be applied in the MTL scenario in order to save the labelling cost. 

\section*{Appendix}

\subsection*{root computation of the Log-Determinant function}

To solve the log-determinant function, we set the derivative of Eq. (\ref{OptimalSigma_log}) for each $\{\sigma_i\}_{i=1}^r \geq 0$ to zero with $\rho = \eta_1\lambda_1$,
\bqs\notag
\frac{1}{\rho}\sigma_i^3 - \frac{1}{\rho}\hat{\sigma}_i\sigma_i^2 + (\frac{1}{\rho} + 2)\sigma_i - \frac{1}{\rho}\hat{\sigma}_i = 0.
\eqs
Assume that $a = \frac{1}{\rho}$, $b = - \frac{1}{\rho}\hat{\sigma}_i$, $c = \frac{1}{\rho} + 2$ and $d = - \frac{1}{\rho}\hat{\sigma}_i$, we define that
$\Delta = \alpha^2 + \beta^3$, where
\bqs\notag
 \alpha = \frac{bc}{6a^2} - \frac{b^3}{27a^3} - \frac{d}{2a}, & & \beta = \frac{c}{3a} - \frac{b^2}{9a^2}.
\eqs
Then the three possible roots of the above cubic equation include one real root and two complex roots,
\bqs\notag
 \sigma_i^{(1)} = & -\frac{b}{3a} + \sqrt[3]{\alpha + \sqrt{\Delta}} + \sqrt[3]{\alpha - \sqrt{\Delta}}, \\
 \sigma_i^{(2)} = & -\frac{b}{3a} + \frac{\sqrt{3}i - 1}{2}\sqrt[3]{\alpha + \sqrt{\Delta}} - \frac{\sqrt{3}i + 1}{2}\sqrt[3]{\alpha - \sqrt{\Delta}}, \\
 \sigma_i^{(3)} = & -\frac{b}{3a} - \frac{\sqrt{3}i + 1}{2}\sqrt[3]{\alpha + \sqrt{\Delta}} + \frac{\sqrt{3}i - 1}{2}\sqrt[3]{\alpha - \sqrt{\Delta}}.
\eqs
According to the $\Delta$, three different scenarios of roots are as follows,
\begin{itemize}
  \item if $\Delta > 0$, the equation has a real root and a conjugate pair of complex roots
  \item if $\Delta = 0$, the equation has three real roots:
  \begin{enumerate}
    \item when $\alpha^2 = \beta^3 = 0$, there are three equal real roots;
    \item when $\alpha^2 = - \beta^3 \neq 0$, there are two equal real roots.
  \end{enumerate}
  \item if $\Delta < 0$, there are three unequal real roots.
\end{itemize}

{
\bibliographystyle{IEEEtran}
\bibliography{OMTC-extend-arxiv}

\begin{thebibliography}{10}
\providecommand{\url}[1]{#1}
\csname url@samestyle\endcsname
\providecommand{\newblock}{\relax}
\providecommand{\bibinfo}[2]{#2}
\providecommand{\BIBentrySTDinterwordspacing}{\spaceskip=0pt\relax}
\providecommand{\BIBentryALTinterwordstretchfactor}{4}
\providecommand{\BIBentryALTinterwordspacing}{\spaceskip=\fontdimen2\font plus
\BIBentryALTinterwordstretchfactor\fontdimen3\font minus
  \fontdimen4\font\relax}
\providecommand{\BIBforeignlanguage}[2]{{%
\expandafter\ifx\csname l@#1\endcsname\relax
\typeout{** WARNING: IEEEtran.bst: No hyphenation pattern has been}%
\typeout{** loaded for the language `#1'. Using the pattern for}%
\typeout{** the default language instead.}%
\else
\language=\csname l@#1\endcsname
\fi
#2}}
\providecommand{\BIBdecl}{\relax}
\BIBdecl

\bibitem{Caruana97}
R.~Caruana, ``Multitask learning,'' \emph{Machine Learning}, vol.~28, no.~1,
  pp. 41--75, 1997.

\bibitem{Evgeniou05}
T.~Evgeniou, C.~A. Micchelli, and M.~Pontil, ``Learning multiple tasks with
  kernel methods,'' \emph{JMLR}, vol.~6, pp. 615--637, 2005.

\bibitem{WATR2010}
C.~Widmer, Y.~Altun, N.~C. Toussaint, and G.~R�tsch, ``Inferring latent task
  structure for multi-task learning via multiple kernel learning,'' \emph{BMC
  Bioinformatics}, vol.~11, no. Suppl 8, p.~S5, 2010.

\bibitem{qi2010semi}
Y.~Qi, O.~Tastan, J.~G. Carbonell, J.~Klein-Seetharaman, and J.~Weston,
  ``Semi-supervised multi-task learning for predicting interactions between
  hiv-1 and human proteins,'' \emph{Bioinformatics}, vol.~26, no.~18, pp.
  i645--i652, 2010.

\bibitem{pan2010transfer}
W.~Pan, E.~W. Xiang, N.~N. Liu, and Q.~Yang, ``Transfer learning in
  collaborative filtering for sparsity reduction.'' in \emph{AAAI}, vol.~10,
  2010, pp. 230--235.

\bibitem{anderson2008theory}
T.~Anderson, \emph{The theory and practice of online learning}.\hskip 1em plus
  0.5em minus 0.4em\relax Athabasca University Press, 2008.

\bibitem{Saha}
A.~Saha, P.~Rai, H.~{Daum\'e III}, and S.~Venkatasubramanian, ``Online learning
  of multiple tasks and their relationships,'' in \emph{AISTATS}, Ft.
  Lauderdale, Florida, 2011.

\bibitem{lugosi2009online}
G.~Lugosi, O.~Papaspiliopoulos, and G.~Stoltz, ``Online multi-task learning
  with hard constraints,'' \emph{arXiv preprint arXiv:0902.3526}, 2009.

\bibitem{ruvolo2014online}
P.~Ruvolo and E.~Eaton, ``Online multi-task learning via sparse dictionary
  optimization,'' in \emph{AAAI-14}, 2014.

\bibitem{attenberg2009collaborative}
J.~Attenberg, K.~Weinberger, A.~Dasgupta, A.~Smola, and M.~Zinkevich,
  ``Collaborative email-spam filtering with the hashing trick,'' in
  \emph{Proceedings of the Sixth Conference on Email and Anti-Spam}, 2009.

\bibitem{yang2016learning}
P.~Yang, G.~Li, P.~Zhao, X.~Li, and S.~Das~Gollapalli, ``Learning correlative
  and personalized structure for online multi-task classification,'' in
  \emph{Proceedings of the 2016 SIAM International Conference on Data
  Mining}.\hskip 1em plus 0.5em minus 0.4em\relax SIAM, 2016, pp. 666--674.

\bibitem{CavallantiCG10}
G.~Cavallanti, N.~Cesa-Bianchi, and C.~Gentile, ``Linear algorithms for online
  multitask classification,'' \emph{JMLR}, vol.~11, pp. 2901--2934, 2010.

\bibitem{gong2012robust}
P.~Gong, J.~Ye, and C.~Zhang, ``Robust multi-task feature learning,'' in
  \emph{ACM SIGKDD (2012)}, 2012, pp. 895--903.

\bibitem{baxter2000model}
J.~Baxter, ``A model of inductive bias learning,'' \emph{J. Artif. Intell.
  Res.(JAIR)}, vol.~12, pp. 149--198, 2000.

\bibitem{bakker2003task}
B.~Bakker and T.~Heskes, ``Task clustering and gating for bayesian multitask
  learning,'' \emph{JMLR}, vol.~4, pp. 83--99, 2003.

\bibitem{yu2005learning}
K.~Yu, V.~Tresp, and A.~Schwaighofer, ``Learning gaussian processes from
  multiple tasks,'' in \emph{ICML}.\hskip 1em plus 0.5em minus 0.4em\relax ACM,
  2005, pp. 1012--1019.

\bibitem{AndoZ05}
R.~K. Ando and T.~Zhang, ``A framework for learning predictive structures from
  multiple tasks and unlabeled data,'' \emph{JMLR}, vol.~6, pp. 1817--1853,
  2005.

\bibitem{Evgeniou04}
T.~Evgeniou and M.~Pontil, ``Regularized multi--task learning,'' in \emph{ACM
  SIGKDD (KDD04)}.\hskip 1em plus 0.5em minus 0.4em\relax New York, NY, USA:
  ACM, 2004, pp. 109--117.

\bibitem{pong2010trace}
T.~K. Pong, P.~Tseng, S.~Ji, and J.~Ye, ``Trace norm regularization:
  Reformulations, algorithms, and multi-task learning,'' \emph{SIAM Journal on
  Optimization}, vol.~20, no.~6, pp. 3465--3489, 2010.

\bibitem{negahban2011estimation}
S.~Negahban and M.~J. Wainwright, ``Estimation of (near) low-rank matrices with
  noise and high-dimensional scaling,'' \emph{The Annals of Statistics}, pp.
  1069--1097, 2011.

\bibitem{zhou2012mutal}
J.~Zhou, J.~Chen, and J.~Ye, \emph{MALSAR: Multi-tAsk Learning via StructurAl
  Regularization}, Arizona State University, 2011.

\bibitem{argyriou2008convex}
A.~Argyriou, T.~Evgeniou, and M.~Pontil, ``Convex multi-task feature
  learning,'' \emph{Machine Learning}, vol.~73, no.~3, pp. 243--272, 2008.

\bibitem{yang2009heterogeneous}
X.~Yang, S.~Kim, and E.~P. Xing, ``Heterogeneous multitask learning with joint
  sparsity constraints,'' in \emph{NIPS}, 2009, pp. 2151--2159.

\bibitem{ArgyriouEP06}
A.~Argyriou, T.~Evgeniou, and M.~Pontil, ``Multi-task feature learning,'' in
  \emph{NIPS}, 2006, pp. 41--48.

\bibitem{AbernethyBR07}
J.~Abernethy, P.~L. Bartlett, and A.~Rakhlin, ``Multitask learning with expert
  advice,'' in \emph{COLT}, 2007, pp. 484--498.

\bibitem{Agarwal}
A.~Agarwal, A.~Rakhlin, and P.~Bartlett, ``Matrix regularization techniques for
  online multitask learning,'' EECS Department, University of California,
  Berkeley, Tech. Rep., Oct 2008.

\bibitem{yan2014multitask}
Y.~Yan, E.~Ricci, R.~Subramanian, G.~Liu, and N.~Sebe, ``Multitask linear
  discriminant analysis for view invariant action recognition,'' \emph{IEEE
  Transactions on Image Processing}, vol.~23, no.~12, pp. 5599--5611, 2014.

\bibitem{haideceptive}
Z.~Hai, P.~Zhao, P.~Cheng, P.~Yang, X.-L. Li, G.~Li, and A.~Financial,
  ``Deceptive review spam detection via exploiting task relatedness and
  unlabeled data,'' in \emph{EMNLP}, 2016.

\bibitem{yan2016multi}
Y.~Yan, E.~Ricci, R.~Subramanian, G.~Liu, O.~Lanz, and N.~Sebe, ``A multi-task
  learning framework for head pose estimation under target motion,''
  \emph{TPAMI}, vol.~38, no.~6, pp. 1070--1083, 2016.

\bibitem{yang2015min}
P.~Yang and P.~Zhao, ``A min-max optimization framework for online graph
  classification,'' in \emph{Proceedings of the 24th ACM International on
  Conference on Information and Knowledge Management}.\hskip 1em plus 0.5em
  minus 0.4em\relax ACM, 2015, pp. 643--652.

\bibitem{conf/icml/ZhaoHJY11}
P.~Zhao, S.~C.~H. Hoi, R.~Jin, and T.~Yang, ``Online {AUC} maximization,'' in
  \emph{ICML-11}, 2011, pp. 233--240.

\bibitem{yang2016efficient}
P.~Yang, P.~Zhao, Z.~Hai, W.~Liu, S.~C. Hoi, and X.-L. Li, ``Efficient
  multi-class selective sampling on graphs,'' in \emph{Proceedings of the
  Thirty-Second Conference on Uncertainty in Artificial Intelligence}.\hskip
  1em plus 0.5em minus 0.4em\relax AUAI Press, 2016, pp. 805--814.

\bibitem{Dekel}
O.~Dekel, P.~M. Long, and Y.~Singer, ``Online learning of multiple tasks with a
  shared loss,'' \emph{JMLR}, vol.~8, pp. 2233--2264, October 2007.

\bibitem{DekelLS06}
------, ``Online multitask learning,'' in \emph{COLT}, 2006, pp. 453--467.

\bibitem{LiCHLJ11}
G.~Li, K.~Chang, S.~C.~H. Hoi, W.~Liu, and R.~Jain, ``Collaborative online
  learning of user generated content,'' in \emph{CIKM}, 2011, pp. 285--290.

\bibitem{Crammer}
K.~Crammer, O.~Dekel, J.~Keshet, S.~Shalev-Shwartz, and Y.~Singer, ``Online
  passive-aggressive algorithms,'' \emph{JMLR}, vol.~7, pp. 551--585, 2006.

\bibitem{crammer2012learning}
K.~Crammer and Y.~Mansour, ``Learning multiple tasks using shared hypotheses,''
  in \emph{Advances in Neural Information Processing Systems}, 2012, pp.
  1475--1483.

\bibitem{murugesan2016adaptive}
K.~Murugesan, H.~Liu, J.~Carbonell, and Y.~Yang, ``Adaptive smoothed online
  multi-task learning,'' in \emph{Advances in Neural Information Processing
  Systems}, 2016, pp. 4296--4304.

\bibitem{basri2003lambertian}
R.~Basri and D.~W. Jacobs, ``Lambertian reflectance and linear subspaces,''
  \emph{TPAMI}, vol.~25, no.~2, pp. 218--233, 2003.

\bibitem{liu2013robust}
G.~Liu, Z.~Lin, S.~Yan, J.~Sun, Y.~Yu, and Y.~Ma, ``Robust recovery of subspace
  structures by low-rank representation,'' \emph{TPAMI}, vol.~35, no.~1, pp.
  171--184, 2013.

\bibitem{favaro2011closed}
P.~Favaro, R.~Vidal, and A.~Ravichandran, ``A closed form solution to robust
  subspace estimation and clustering,'' in \emph{Computer Vision and Pattern
  Recognition (CVPR), 2011 IEEE Conference on}.\hskip 1em plus 0.5em minus
  0.4em\relax IEEE, 2011, pp. 1801--1807.

\bibitem{shalev2011stochastic}
S.~Shalev-Shwartz and A.~Tewari, ``Stochastic methods for l 1-regularized loss
  minimization,'' \emph{JMLR}, vol.~12, pp. 1865--1892, 2011.

\bibitem{bertsekas1999nonlinear}
D.~P. Bertsekas, ``Nonlinear programming,'' 1999.

\bibitem{rockafellar1976monotone}
R.~T. Rockafellar, ``Monotone operators and the proximal point algorithm,''
  \emph{SIAM journal on control and optimization}, vol.~14, no.~5, pp.
  877--898, 1976.

\bibitem{bregman1967relaxation}
L.~M. Bregman, ``The relaxation method of finding the common point of convex
  sets and its application to the solution of problems in convex programming,''
  \emph{USSR computational mathematics and mathematical physics}, vol.~7,
  no.~3, pp. 200--217, 1967.

\bibitem{beck2003mirror}
A.~Beck and M.~Teboulle, ``Mirror descent and nonlinear projected subgradient
  methods for convex optimization,'' \emph{Operations Research Letters},
  vol.~31, no.~3, pp. 167--175, 2003.

\bibitem{kim2010tree}
S.~Kim and E.~P. Xing, ``Tree-guided group lasso for multi-task regression with
  structured sparsity,'' in \emph{ICML-10}, 2010, pp. 543--550.

\bibitem{chen2012learning}
J.~Chen, J.~Liu, and J.~Ye, ``Learning incoherent sparse and low-rank patterns
  from multiple tasks,'' \emph{TKDD}, vol.~5, no.~4, p.~22, 2012.

\bibitem{vandenberghe1996semidefinite}
L.~Vandenberghe and S.~Boyd, ``Semidefinite programming,'' \emph{SIAM review},
  vol.~38, no.~1, pp. 49--95, 1996.

\bibitem{boyd2004convex}
S.~Boyd and L.~Vandenberghe, \emph{Convex optimization}.\hskip 1em plus 0.5em
  minus 0.4em\relax Cambridge university press, 2004.

\bibitem{boyd2011distributed}
S.~Boyd, N.~Parikh, E.~Chu, B.~Peleato, and J.~Eckstein, ``Distributed
  optimization and statistical learning via the alternating direction method of
  multipliers,'' \emph{Foundations and Trends{\textregistered} in Machine
  Learning}, vol.~3, no.~1, pp. 1--122, 2011.

\bibitem{tibshirani1996regression}
R.~Tibshirani, ``Regression shrinkage and selection via the lasso,''
  \emph{Journal of the Royal Statistical Society. Series B (Methodological)},
  pp. 267--288, 1996.

\bibitem{candes2009exact}
E.~J. Cand{\`e}s and B.~Recht, ``Exact matrix completion via convex
  optimization,'' \emph{Foundations of Computational mathematics}, vol.~9,
  no.~6, pp. 717--772, 2009.

\bibitem{candes2010power}
E.~J. Cand{\`e}s and T.~Tao, ``The power of convex relaxation: Near-optimal
  matrix completion,'' \emph{Information Theory, IEEE Transactions on},
  vol.~56, no.~5, pp. 2053--2080, 2010.

\bibitem{duchi2010composite}
J.~C. Duchi, S.~Shalev-Shwartz, Y.~Singer, and A.~Tewari, ``Composite objective
  mirror descent.'' in \emph{COLT}, 2010, pp. 14--26.

\bibitem{yang2015aggressive}
P.~Yang, P.~Zhao, V.~W. Zheng, and X.-L. Li, ``An aggressive graph-based
  selective sampling algorithm for classification,'' in \emph{Data Mining
  (ICDM), 2015 IEEE International Conference on}.\hskip 1em plus 0.5em minus
  0.4em\relax IEEE, 2015, pp. 509--518.

\bibitem{yang2014ensemble}
P.~Yang, X.~Li, H.-N. Chua, C.-K. Kwoh, and S.-K. Ng, ``Ensemble positive
  unlabeled learning for disease gene identification,'' \emph{PloS one},
  vol.~9, no.~5, p. e97079, 2014.

\bibitem{yang2014ldsplit}
P.~Yang, M.~Wu, J.~Guo, C.~K. Kwoh, T.~M. Przytycka, and J.~Zheng, ``Ldsplit:
  screening for cis-regulatory motifs stimulating meiotic recombination
  hotspots by analysis of dna sequence polymorphisms,'' \emph{BMC
  bioinformatics}, vol.~15, no.~1, p.~48, 2014.

\end{thebibliography}
}

\begin{IEEEbiography}
{Peng Yang} received his PhD from the School of Computer Engineering at the Nanyang Technological University, Singapore. He is currently a research fellow in the Computer, Electrical and Mathematical Sciences and Engineering Division at King Abdullah University of Science and Technology (KAUST), Saudi Arabia. He was a research scientist with Institute for Infocomm Research (I2R), A*STAR, Singapore from 2013 to 2016 and a research scientist with Tencent AI Lab, China in 2017. His research interests are machine learning, data mining, and bioinformatics.
\end{IEEEbiography}
\vspace{-0.1in}

\begin{IEEEbiography}
{Peilin Zhao} is currently a Senior Algorithm Expert in Ant Financial service group. His research Interests are Machine Learning and its applications to Big Data Analytics, etc. He previously worked at A*STAR, Singapore,  Baidu Research China,  Rutgers University USA. He received his PHD degree from  Nanyang Technological University and his bachelor degree from Zhejiang University.
\end{IEEEbiography}
\vspace{-0.1in}

\begin{IEEEbiography}
{Xin Gao} received the BS degree in computer science from Tsinghua University, in 2004, and the PhD degree in computer science from University of Waterloo, in 2009. He is currently an associate professor of computer science in the Computer, Electrical and Mathematical Sciences and Engineering Division at King Abdullah University of Science and Technology (KAUST), Saudi Arabia. Prior to joining KAUST, he was a Lane Fellow at Lane Center for Computational Biology in School of Computer Science at Carnegie Mellon University. His research interests are machine learning and bioinformatics.
\end{IEEEbiography}

\end{document}